\documentclass[12pt,3p]{elsarticle}

\usepackage{lineno,hyperref}
\modulolinenumbers[5]

\usepackage{amsthm}
\usepackage{verbatim}
\usepackage[utf8]{inputenc}
\usepackage{amssymb}
\usepackage{amsmath}
\usepackage{xspace}
\usepackage{xcolor}
\usepackage{algorithm}
\usepackage{algpseudocode}
\usepackage{graphicx}
\usepackage{mathtools}
\usepackage{booktabs} % For formal tables
\usepackage{thmtools} 
\usepackage{thm-restate}
\usepackage{algorithmicx}
\usepackage{algpseudocode}

\newcommand{\confx}{\ensuremath{h}}
\newcommand{\confy}{\ensuremath{H}}

\DeclareMathOperator*{\argmin}{arg\,min}
\newcommand{\oneoneea}{$(1+1)~\text{EA}$\xspace}
\newcommand{\oneoneais}{$(1+1)~\text{IA}^{hyp}$ } 
\newcommand{\oneoneeaageing}{$(1+1)~\text{EA}^{ageing}$ }
\newcommand{\muoneeaageing}{$(\mu+1)~\text{EA}^{ageing}$ }
\newcommand{\partition}{\textsc{Partition}\xspace}
\newcommand{\prob}[1]{\text{Pr}\left\{#1\right\}}

\newcommand{\X}{\ensuremath{{\mathcal{X}}}}
\newcommand{\WittGeneralisedInstance}{$G^*_{\epsilon}$\xspace}

\newtheorem{theorem}{Theorem}
\newtheorem{corollary}{Corollary}
\newtheorem{property}{Property}
\newtheorem{lemma}{Lemma}
\newtheorem{definition}{Definition}

\journal{Journal of \LaTeX\ Templates}

\begin{document}

\begin{frontmatter}

\title{Artificial Immune Systems Can Find Arbitrarily Good Approximations for the NP-Hard Number Partitioning Problem}
\tnotetext[mytitlenote]{ {\color{black}An extended abstract of this paper without proofs was presented at PPSN 2018~\cite{CorusOlivetoYazdani2018}.}}

\author{Dogan Corus}
\ead{d.corus@sheffield.ac.uk}
\author{Pietro S. Oliveto}
\ead{p.oliveto@sheffield.ac.uk}
\author{Donya Yazdani}
\ead{dyazdani1@sheffield.ac.uk}
\address{Rigorous Research, Department of Computer Science, University of Sheffield}
\address{Sheffield, UK}
\address{S1 4DP}

\begin{abstract}
Typical artificial immune system (AIS) operators such as hypermutations with mutation potential and ageing allow 
to efficiently overcome local optima from which evolutionary algorithms (EAs) struggle to escape.
Such behaviour has been shown for artificial example functions constructed especially  to show difficulties that EAs may encounter during the optimisation process.
{\color{black}However, no evidence is available indicating that these two operators have similar behaviour also in more realistic problems.}
In this paper we perform an analysis for the standard NP-hard \partition problem from combinatorial optimisation and rigorously show that
hypermutations and ageing allow AISs to efficiently escape from local optima where standard EAs require exponential time. As a result we prove that while EAs and random local search (RLS) may get trapped on 4/3 approximations, 
AISs find arbitrarily good approximate solutions of ratio (1+$\epsilon$) 
{\color{black}within $n(\epsilon ^{-(2/\epsilon)-1})(1-\epsilon)^{-2} e^{3} 
2^{2/\epsilon} + 2n^3 2^{2/\epsilon} + 2n^3$ function evaluations in 
expectation. This expectation is polynomial 
in the problem size and exponential only in $1/\epsilon$}.

\end{abstract}

%\begin{keyword}
%Artificial immune systems, NP-hard \partition problem, Combinatorial optimisation 
%\end{keyword}

\end{frontmatter}

%\linenumbers


%\ccsdesc[500]{Computer systems organization~Embedded systems}

%\keywords{ACM proceedings, \LaTeX, text tagging}

\section{Introduction}

Artificial immune systems (AISs) take inspiration from the immune system of vertebrates to solve complex computational problems.
Given the role of the natural immune system to recognise and protect the 
organism from viruses and bacteria, natural applications of AISs have been 
pattern recognition, %\cite{PatternRecogApp}, 
computer security, virus detection and anomaly detection \cite{ComputerSecurityApp,VirusDetectApp,AnamolyDetectApp}. Various AISs, inspired by Burnet's clonal selection principle \cite{Burnet1959}, have been devised for solving optimisation problems.
Amongst these, the most popular are Clonalg~\cite{decastro}, the B-Cell algorithm~\cite{kelsey} and Opt-IA~\cite{CutelloTEVC}. 
 
AISs for optimisation are very similar to evolutionary algorithms (EAs) since 
they essentially use the same Darwinian evolutionary principles 
to evolve populations of solutions 
(here called antibodies). In particular, they use the same natural selection principles to gradually evolve high quality solutions.
The main distinguishing feature of AISs to more classical EAs is their use of variation operators that typically have higher mutation rates
compared to the standard bit mutations (SBM) of EAs where each bit is mutated independently with {\color{black}(typically small)} probability $p$. Examples are the contiguous somatic mutations (CSM) of the B-Cell algorithm and the hypermutations with mutation potential of Opt-IA.
Another distinguishing feature is their use of ageing operators that remove old solutions which have spent a long time without improving (local optima).
{\color{black} However, it is still largely unclear} on what problems an AIS will have better performance to that of EAs.
Also very little guidance is available on when a class of AISs should be applied rather than another.
{\color{black}Amongst the few available results, it has been proven that there exist instance classes of both the longest common subsequence \cite{JansenZarges2012} and the NP-hard vertex cover~\cite{JansenOlivetoZarges2011} problems which are difficult for EAs equipped with SBM and crossover, i.e., they require exponential expected time to find the global optimum, while the B-Cell algorithm locates it efficiently.}
The superior performance is due to the ability of the contiguous somatic mutations of the B-Cell algorithm to efficiently escape the local optima of these instances while SBM require
exponential expected time in the size of the instance. 

Apart from these results, the theoretical understanding of AISs relies on analyses of their behaviour for artificially constructed toy problems.
Recently it has been shown how both the hypermutations with mutation potential and the ageing operator of Opt-IA can lead to considerable speed-ups compared to the performance 
of well-studied EAs using SBM for standard benchmark functions used in the evolutionary computation community such as \textsc{Jump}, \textsc{Cliff} or \textsc{Trap}~\cite{CorusOlivetoYazdani2017}.
While the performance of hypermutation operators to escape the local optima of these functions is comparable to that of the EAs with high mutation rates that have been increasingly gaining popularity since 2009~\cite{DoerrGecco2017,OlivetoLehreNeumann2009,CorusOlivetoTEVC,JumpTEVC2017,DoerrDoerrEbel2015}, ageing allows the optimisation of hard instances of \textsc{Cliff} in $O(n \log n)$ where $n$ is the problem size. Such 
a runtime is required by all unbiased unary (mutation-based) randomised search heuristics to optimise any function with unique optimum~\cite{LehrWitt2012}. Hence, while the expected runtime for SBM algorithms is exponential for these instances of \textsc{Cliff}, the \muoneeaageing is asymptotically as fast as possible. Interestingly, a similar result holds also for the CSM of B-Cell: SBM requires exponential time to optimise the easiest function for CSM (i.e., \textsc{MinBlocks} )\cite{EasiestFunctions}.
% , i.e., 
% SBM and local search require $\Omega(n \log n)$ function evaluations to optimise their easiest function with unique 
% optimum {\color{black}--} \textsc{OneMax}~\cite{DrosteJansenWegener2002}.
%
Although some of these speed-ups over standard EA performance are particularly impressive, no similar evidence of superior performance of these two operators is available for
more realistic problems. 
{\color {black}
In this paper we perform an analysis for \textsc{Partition}, also known as Number Partitioning, which is considered one of the six basic NP-complete problems~\cite{GareyJohnson1979}. %and is proved to be NP-hard~\cite{Karp1972,BrunoCoffmanSethi1974}. 
%Some of its applications include multiprocessor scheduling \cite{CoffmanLueker1991} and public key cryptography \cite{MerkleHellman1978}.}
\textsc{Partition} as a decision making and an optimisation process arises in many resource allocation tasks in manufacturing, production and information processing systems~\cite{Pinedo2016, Hayes2002}. 

We are {\color{black}particularly} concerned with comparing the performance of AISs with other general purpose algorithms for the \textsc{Partition} problem. Regarding such algorithms, the performance of  RLS and the (1+1)~EA is well understood in the literature~\cite{Witt2005,NeumannWitt2015}. 
%Their analysis, which is based on a simulation of Graham's PTAS~\cite{Graham1969}, 
It has been shown that RLS and EAs using SBM may get stuck on local optima which lead to a worst case approximation ratio of 4/3. In order to achieve a (1+$\epsilon$) approximation for arbitrary $\epsilon$, a clever restart strategy has to be put in place.}
Herein, we first show the power of hypermutations and ageing  by proving that each of them solve to optimality
instances that are hard for RLS and SBM EAs, by efficiently escaping local optima for which the EAs struggle. Afterwards we prove that AISs using hypermutations 
with mutation potential guarantee arbitrarily good solutions of approximation 
ratio $(1+\epsilon)$ for any $\epsilon=\omega(n^{-1/2})$ in expected $n(\epsilon ^{-(2/\epsilon)-1})(1-\epsilon)^{-2} e^{3} 2^{2/\epsilon} + 2n^3 2^{2/\epsilon} + 2n^3$ fitness function evaluations, which reduces to $O(n^{3})$ for any constant $\epsilon$
without requiring any restarts. On the other hand, we prove that an AIS with SBM 
and ageing can efficiently achieve the  same approximation ratio up to $\epsilon \geq 4/n$ in $O(n^2)$ expected fitness function evaluations, 
automatically restarting the optimisation process by implicitly detecting when 
it is stuck on a local optimum. In contrast, if an appropriate restart strategy is set up, RLS and the \oneoneea achieve the same approximation ratio (up to $\epsilon=\Omega(1/n)$) in ${O(n \ln(1/\epsilon))\cdot2^{(e \log{e}+e)\lceil 2/\epsilon \rceil \ln(4/\epsilon)+O(1/\epsilon)}}$ expected function evaluations. 
This is only possible if  the restart strategy consists of  $2^{\Omega((1/\epsilon) \ln (1/\epsilon))}$ expected restarts of  $O(n \log (1/\epsilon))$ evaluations each. Hence, it is only possible if the desired approximation ratio is decided in advance. On the other hand, none of this information is required for either of the two AISs to efficiently identify the same approximation ratios.
To the best of our knowledge this is the first time {\color{black}that} either 
hypermutations or ageing have been theoretically analysed for a standard problem 
from combinatorial optimisation and the first time performance guarantees of any 
AIS are proven in such a setting.
% The rest of the paper is structured as follows.
% In the Preliminaries we introduce the \textsc{Partition} problem, the AIS operators and the simple AIS algorithms considered in the paper.
% In Section \ref{sec:generalised} we show how hypermutations and ageing each easily solve to optimality the instance class where EAs and RLS get stuck on an expected 4/3
% approximation ratio. To this end we generalise the class to include a much larger set of instances where AIS require polynomial time versus the exponential expected runtime of SBM or RLS.
% %In Section \ref{sec:uniqueoptimum} we introduce a more complicated multimodal instance class where the number of local optima is a parameter of the class.
% %We prove that while hypermutations with mutation potential are efficient for the instance class, SBM and RLS are inefficient.
% In Section \ref{sec:approx} we prove that hypermutations and ageing guarantee the (1+$\epsilon$) expected approximation ratio without requiring any predefined restart strategy.
% We conclude the paper with a  final discussion of the results.
%Some proofs are in an Appendix for the convenience of the reviewers.

\section{Preliminaries}
\begin{algorithm}[t]
    \caption{(1+1)~IA$^{hyp}$~\cite{CorusOlivetoYazdani2017} for minimisation}
    \label{alg:hyp}
    \begin{algorithmic}[1] % The number tells where the line numbering should start
        
            \State Set each bit in $x$ to 1 with probability $1/2$ and to 0 otherwise.
            \State Evaluate $f(x)$.      
            \While{termination condition not satisfied} %\Comment{}
            \State  {$y := x$;}
            \State {$F:=\{1,\dots,n\}$;}
            \While{$F \neq \emptyset$ and $f(y) \geq f(x)$}
                \State $i :=$ Sample $F$ {\color{black}uniformly at 
random};
				\State {$F:=F \setminus \{i\}$;}
				\State {Flip $y_i$;}
				\State {Evaluate $f(y)$;}
           \EndWhile     
                \If {$f(y)\leq f(x)$} 
                \State {$x:=y$;}
                \EndIf
            \EndWhile
    \end{algorithmic}
\end{algorithm}

\subsection{Artificial Immune Systems and Evolutionary Algorithms}

%We aim to analyse the performance of two typical AIS operators for the NP-Hard \partition problem.
%To this end we will incorporate them into minimal algorithmic frameworks.
{\color{black}AISs designed for optimisation are inspired by the optimisation processes in the adaptive immune system of vertebrates. The  clonal selection principle is a central paradigm of the adaptive immune system describing how antibodies are produced in response to antigens.
Proposed by Burnet in the 1950s~\cite{Burnet1959}, the clonal selection principle states that antibodies are produced by b-cells in each host only in response to a specific antigen to which the host is exposed to. When a b-cell binds to an antigen, it becomes activated and it generates many identical b-cells called clones. The next step in the clonal selection is a process called \textit{affinity maturation}. The aim of this process is to increase the affinity of the antibodies with that specific antigen. This is achieved by applying somatic hypermutation (mutations with high rates) to the b-cells. The b-cells with greater affinity with the antigen survive and are used in the next generation of the process~\cite{Janeway2011}.
 {\it Hypermutation with mutation potential} operators are inspired by such high mutation rates~\cite{CutelloTEVC}.}
 For the purpose of optimisation, these high mutation rates may allow the algorithm to escape local optima by identifying promising search areas far away from the current ones. 
%The role of the mutation potential $M$, which is typically linear in the problem size (i.e., $M=cn$), varies in different variants of the operator. 
%In {\it inversely proportional} hypermutations\cite{JansenZarges2011} and in {\it static} hypermutations it indicates an upper bound on the number of bits that will be flipped in one mutation operation. In {\it proportional} hypermutations a higher number of bits, proportional to the fitness, may flip~\cite{CutelloTEVC}.

AISs for optimisation are applied in the same way as other well-known randomised search heuristics. In essence, all that is required is a fitness function to evaluate how good solutions are and some way to represent candidate solutions to the problem (i.e., representation). The main difference between AISs and other heuristic search algorithms is the metaphor from which they are inspired and consequent variation and diversity operators. In this paper we will use bit-strings  (i.e., $x \in\{0,1\}^n$) to represent  candidate solutions and we will analyse the static hypermutation operator considered in~\cite{CorusOlivetoYazdani2017}
for benchmark functions, where the maximum number of bits to be flipped is fixed 
to {\color{black}$cn$ for some constant $c\in (0,1]$} throughout the 
optimisation process.
Two variants of static hypermutations have been proposed in the literature. A straightforward version, where in each mutation exactly $cn$ bits are flipped, and another one 
called {\it stop at first constructive mutation} (FCM)
where the solution quality is evaluated after each of the $cn$ bit-flips and 
the operator is halted once a {\it constructive mutation} occurs.
Since Corus et al.~\cite{CorusOlivetoYazdani2017} proved that the straightforward version requires exponential time to optimise any function with a polynomial number of optima, we will consider the version with FCM.
We define a mutation to be {\it constructive} if the solution is strictly better than the original parent  and we set $c=1$
such that all bits will flip if no constructive mutation is found before.
For the sake of understanding the potentiality of the operator, we embed it into a minimal AIS framework that uses only one antibody (or individual) and creates a new one in each iteration
via hypermutation as done previously in the literature~\cite{CorusOlivetoYazdani2017} (the algorithm is essentially a (1+1)~EA~\cite{DrosteJansenWegener2002} that uses hypermutations instead of SBM).
The simple AIS for the minimisation of objective functions, called \oneoneais for consistency with the evolutionary computation literature, is formally described in Algorithm~\ref{alg:hyp}.

\begin{algorithm}[h!]
    \caption{\muoneeaageing \cite{OlivetoSudholt2014} for minimisation}
    \label{alg:ageing}
    \begin{algorithmic}[1] % The number tells where the line numbering should start 
            \State Create population $P:=\{x_1, \cdots, x_{\mu}\}$ with each bit in $x_i$ set to 1 with probability $1/2$ and to 0 otherwise;
            \State For all $x \in P$ evaluate $f(x)$ and set 
{\color{black}$x^{age}:=0$}. 
            \While{termination condition not satisfied} %\Comment{}
            	 \For {all $x \in P$}
            	 \State Set $x^{age}:=x^{age}+1$;
            	 \EndFor
            	\State Select $x \in P$ uniformly at random;
            	\State $y:=x$;
            	\State{Flip each bit in $y$ with probability $1/n$;} 
            	\State{Add $y$ to $P$;}                   
                \If{$f(y)< f(x)$}
                \State $y^{age}:=0$;
                \Else
                \State $y^{age}:=x^{age};$
				\EndIf
           		\For {all $x \in P$}
           		\If {$x^{age} \geq \tau$} 
				\State{Remove $x$ from $P$;} 
				\EndIf
				\EndFor
				\If {$|P| > \mu$}
				\State{Remove the individual from $P$ with the highest $f(x)$};
				\EndIf
				\If {$|P|< \mu $}
				\State{Add enough randomly created individuals {\color{black}with age 0} to $P$ until $|P|=\mu$};
\EndIf
            \EndWhile
    \end{algorithmic}
\end{algorithm}

\begin{algorithm}[h!]
    \caption{(1+1)~EA~\cite{DrosteJansenWegener2002} for minimisation}
    \label{alg:oneea}
    \begin{algorithmic}[1] % The number tells where the line numbering should start 
            \State Set each bit in $x$ to 1 with probability $1/2$ and to 0 otherwise; 					\State{Evaluate $f(x)$;}   
            \While{the termination condition is not satisfied} %\Comment{}
                \State Create $y$ by flipping each bit of $x$ with probability $1/n$;
                \If {$f(y)\leq f(x)$}
                \State{ $x:=y$;}
                \EndIf
            \EndWhile
    \end{algorithmic}
\end{algorithm}

Another popular operator used in AISs is {\it ageing}.
The idea behind this operator is to remove antibodies which have not improved for a long time.
Intuitively, these antibodies are not improving because they are trapped on some local optimum and they may be obstructing the algorithm from progressing in more promising 
areas of the search space (i.e., the population of antibodies may quickly be taken over by a high quality antibody on a local optimum).
The antibodies that have been removed by the ageing operator are replaced by new ones initialised at random. 

%Three different ageing mechanisms have been proposed in the literature. 
%{\it Static} ageing, where antibodies are removed from the population if they have not improved for $\tau$ generations, 
%{\it Stochastic} ageing, where they are removed in each generation with some probability $p_{die}$ and {\it Hybrid} ageing, where individuals are removed
%with probability $p_{die}$ only once they have failed to improve for at least $\tau$ generations.

Ageing operators have been proven to be very effective at automatically restarting the AIS, without having to set up a restart strategy in advance, once it has converged on a local optimum ~\cite{JansenZarges2011c}.  Stochastic versions have been shown to also allow antibodies to escape from local optima~\cite{OlivetoSudholt2014,CorusOlivetoYazdani2017}.
As in previous analyses we incorporate the ageing operator in a simple population-based evolutionary algorithmic framework, the $(\mu+1)$~EA~\cite{Witt2006}, and for simplicity consider the static variant
where antibodies are removed from the population with probability 1 if they have not improved for $\tau$ generations.
The \muoneeaageing is formally defined in Algorithm \ref{alg:ageing}.
%To the best of our knowledge this is the first time either hypermutations or ageing operators have been theoretically analysed for a standard combinatorial optimisation problem.
%
We will compare the performance of the AISs with the standard \oneoneea and RLS for which the performance for \partition is known.
The former uses standard bit mutation, i.e., it flips each bit of the parent independently with probability $1/n$ in each iteration, while the latter flips exactly one bit.
The \oneoneea for minimisation is formally described in Algorithm~\ref{alg:oneea}.

{\color{black}%While most of the available time complexity analysis are related to toy problems with significant structures (see {\color{red}\cite{}} for surveys of results), 
The 
performance of simple EAs is known for many combinatorial optimisation problems both 
in P and in NP-hard. Recent work shows that EAs provide optimal solutions in expected FPT (fixed parameter tractable) time for 
the NP-hard generalised minimum spanning tree problem~\cite{corus2016}, the Euclidean 
and the generalised traveling salesperson problems~\cite{sutton2014, corus2016} and 
approximation results for various NP-hard problems on scale-free networks~\cite{chauhan2017}. EAs have been shown to be efficient also on problems in P such as the 
all-pairs shortest path problem~\cite{DOERR201312} and some easy instance classes of 
the k-CNF~\cite{doerr2017time} and the knapsack~\cite{neumann2018runtime} 
problems. Although EAs are efficient general purpose solvers, super-polynomial 
lower bounds on a selection of easy problems
have also recently been made available~\cite{sutton2016superpolynomial}. {\color{black} Moreover, rigorous analyses of more realistic standard EAs using populations and crossover are possible nowadays \cite{CorusOlivetoTEVC,JumpTEVC2017,LengelerPPSN2018,LehrTEVC2018}. In particular, the efficient performance of a standard steady-state ($\mu$+1)~GA as an FPT algorithm for the closest substring problem has recently been shown~\cite{Sutton2018}}. We refer the reader to~\cite{NeumannWitt2010} for a comprehensive overview of older results in the field for combinatorial optimisation problems.}
% ,

%{\color{black}
%Both the presentation of our results and our analysis will frequently use Landau notation which denotes the asymptotic order of two functions, e.g., $f(x)=O(g(x))$ implies that $f(x)$ grows at most as fast as $g(x)$ when $x$ goes to infinity. We refer the reader to~\cite{LandauRef} for the precise definitions of the $O$, $o$, $\Omega$, $\omega$ and $\Theta$ notations. 
%}
Both the presentation of our results and our analysis will frequently make use of the following asymptotic notations: $O$, $o$, $\Omega$, $\omega$ and $\Theta$. We use the $O$-notation to show an asymptotic upper bound for a non-negative function $f(n)$ when $n$ goes to infinity. Hence, we say $f(n)$ belongs to the set of functions $O(g(n))$ if there exist positive constants $c$ and $n_0$ such that $0 \leq f(n) \leq cg(n)$ for all $n\geq n_0$. In other words, by using either $f(n) \in O(g(n))$, or equivalently  $f(n) = O(g(n))$,  we asymptotically bound function $f(n)$ from above by $g(n)$ up to a constant factor.
We refer the reader to~\cite{LandauRef} for the definitions of the other notations.

\subsection{Makespan Scheduling and \partition}

{\color{black}The optimisation version of the \partition problem is formally 
defined as follows:
\begin{definition}
Given positive integers $a_1, a_2,\ldots, a_n$, what is 
$\argmin\limits_{B \subseteq [n]} \max{\left(\sum\limits_{i \in 
B}a_i,\sum\limits_{i \in [n]\setminus B }a_i\right)}$?  
\end{definition}
% \begin{definition}
% Given positive processing times $p_1, p_2, \ldots, p_n$ and $m \in \mathbb{N}$. 
% For $i\in [m]$, what are $S_i \subseteq [n]$ 
% which minimises $\max\limits_{i \in [m]}{\left(\sum\limits_{j\in S_i}p_i\right)}$ such that $\bigcup\limits_{i\in[m]} S_i= [n]$?
% \end{definition}

Given $n$ jobs with positive processing times $p_1\geq p_2 \geq \dots \geq p_{n-1} \geq p_n$ and the number of machines $m$, the makespan scheduling problem with $m$ parallel machines is that of scheduling the $n$ jobs on $m$ identical machines in a way that the overall completion time (i.e., the {\it makespan}) is minimised~\cite{Pinedo2016}. Assigning the jobs to 
the machines is equivalent to dividing the set of processing times into $m$
disjoint subsets such that the largest subset sum yields the makespan. When $m=2$ and the 
processing times $p_i$ are scaled up to be integers $p^{*}_{i}$,  the problem is 
equivalent to the \partition problem with the set of integers $a_i=p_{i}^*$. }
This simple to define scheduling problem is well studied in theoretical 
computer science and is known to be NP-hard~\cite{Karp1972,BrunoCoffmanSethi1974}. 
Hence, it cannot be expected
that any algorithm {\color{black}finds optimal solutions} to all instances in 
polynomial time. 

\subsubsection{Problem-specific Algorithms}
{\color{black}The \partition problem is considered an easy NP-hard problem because
% To obtain optimal solutions of \partition problem dynamic programming are usually used. 
there are some dynamic programming algorithms that provide optimal solutions %of the \partition problem  
in pseudo-polynomial time of $O(n \sum{p_i})$~\cite{Pinedo2016}. The general idea behind such algorithms is to create an $n \times \sum{p_i}$ table where each pair $(i,j)$ indicates if there exists  a subset of $i$ that sums up to $j$. Then, by simple backtracking the optimal subset will be found. Clearly, such algorithms would not be practical for large sizes of $n \times \sum{p_i}$.

As the NP-hardness of \partition suggests, not all instances can be optimised in polynomial time unless P$=$NP. Hence, approximation algorithms are used to obtain approximate solutions. For a minimisation problem, an algorithm A which runs in polynomial time with respect to its input size is called a {\it $(1+\epsilon)$ approximation algorithm} if it can guarantee solutions with quality $(1+\epsilon) \cdot f(opt)$ for some $\epsilon>0$ where $f(opt)$ is the quality of an optimal solution. Such algorithm is also called a {\it polynomial approximation scheme} (PTAS) if it can provide a $(1+\epsilon)$ solution for any arbitrarily small $\epsilon>0$. If the runtime is also polynomial in $1/\epsilon$, then Algorithm A is called {\it fully polynomial approximation scheme} (FPTAS).

%However, there exist efficient problem specific algorithms which guarantee 
%solutions with approximation ratio (1+$\epsilon$) {\color{black}(\emph{i.e.}, 
%with an objective function value at most a (1+$\epsilon$) factor larger than the 
%optimal)} in time $O(n^3/\epsilon)$, in classical complexity measures,  which is 
%polynomial both in $n$ and $1/\epsilon$~\cite{Hochbaum1997}.
Since the \partition problem is a special case of makespan scheduling with $m=2$, any approximation algorithm for the latter problem may be applied to \partition. 
The first approximation algorithm for makespan scheduling was presented by Graham in 1969~\cite{Graham1969}. He showed a worst-case approximation ratio of $2-1/m$ where $m$ is the number of machines (this ratio is $3/2$ for \partition). He then proposed a simple heuristic which guarantees a better approximation ratio of $4/3-1/(3m)$ (which is $7/6$ for \partition) with running time of $O(n \log n)$. This greedy scheme, known as {\it longest processing time} (LPT), assigns the job with longest processing time to the emptier machine each time. He then improved the approximation ratio to $1+ \frac{1-1/m}{1+\lceil{k/m}\rceil}$ by using a more costly heuristic. This PTAS first assigns the $k$ large jobs in an optimal way, then distributes the small jobs one by one in any arbitrary order to the emptiest machine. The ideas behind these heuristics are widely used in the later works, including a heuristic proposed by Hochbaum and Shmoys which guarantees a $(1+\epsilon)$ approximation in $O((n/\epsilon)^{1/\epsilon^2})$ for makespan scheduling on an arbitrary number of machines~\cite{HochbaumShmoys1987}. This heuristic first finds an optimal way of assigning large jobs and then distributes the small jobs using LPT. This is roughly the best runtime achievable for makespan scheduling on an arbitrary number of machines \cite{HochbaumShmoys1987}.
 However, for limited number of machines better results are available in the literature.
 For $m=2$ (i.e., \partition),  Ibarra and Kim have shown an FPTAS which runs in $O(n/\epsilon^2)$ and shortly after Sahni showed an FPTAS based on dynamic programming which runs in $O(n^2/\epsilon)$. This runtime is $O(n\cdot(n^2/\epsilon)^{m-1})$ for larger number of machines $m>3$~\cite{Sahni1976}. 
%In the case of two machines, Ibarra and Kim have shown an FPTAS which runs in $O(n/\epsilon^3)$ in 1975~\cite{IbarraKim1975} and later on in 1976, Sahni showed another FPTAS for \partition  which runs in $O(n^2/\epsilon)$~\cite{Sahni1976}. 

%[where to put it?!] Among the most well-known heuristics is the Karmarkar-Karp algorithm or {\it largest differencing method} (LDM)~\cite{KarmarkarKarp1982}, which is shown to have a worst case approximation ratio of $7/6$ for two machines in time $O(n \log n)$~\cite{Fischetti1Martello987}. This algorithm takes two largest jobs from the sorted list of jobs at each stage and replaces them with the difference of their processing times. The processing time of the final remaining job is the discrepanc1y of two machines. By a simple backtracking the partitions will be found. In 1998, Korf proposed a branch-and-bound algorithm using LDM
}

\subsubsection{General Purpose Heuristics}

% [intro to what is a general purpose heuristic?]
For the application of randomised search heuristics a solution may easily be 
represented with a bit string $x\in\{0,1\}^n$ where each bit $i$ represents job 
$i$ and if the bit is set to 0 it is assigned to the first machine ($M_1$) and 
otherwise to the second machine ($M_2$). Hence, the goal is to minimise 
{\color{black} the makespan} $f(x) 
:= \max \big\{\sum_{i=1}^n p_ix_i, \sum_{i=1}^n p_i (1 -x_i)\big\}$, 
i.e., the processing time of the last machine to terminate. 

Both RLS and the \oneoneea have roughly 4/3 worst case expected approximation ratios for 
the problem (i.e., there exist instances where they can get stuck on solutions by 
a factor of $4/3-\epsilon$ worse than the optimal one)~\cite{Witt2005}. 
However if an appropriate restart strategy is setup in advance,
both algorithms may be turned into polynomial randomised approximation schemes (PRAS, i.e., algorithms that compute a (1+$\epsilon$) approximation in polynomial time in the problem size with probability at least $3/4$) with a runtime bounded by ${O(n \ln(1/\epsilon))\cdot2^{(e \log{e}+e)\lceil 2/\epsilon \rceil \ln(4/\epsilon)+O(1/\epsilon)}}$
 ~\cite{Witt2005}.
Hence, as long as $\epsilon$ does not depend on the problem size, the algorithms can achieve arbitrarily good approximations with an appropriate restart strategy. The analysis of the \oneoneea and RLS essentially {\color{black} shows that the algorithms with a proper restart strategy simulate Graham's PTAS.}

In this paper we will show that AISs can achieve stronger results. Firstly, we 
will show that both ageing and hypermutations can efficiently solve to 
optimality the worst-case instances for RLS and the \oneoneea. More importantly, 
we will prove that ageing automatically achieves the necessary restart strategy 
to guarantee the (1+$\epsilon$) approximation while hypermutations
guarantee it in a single run in polynomial expected runtime and {\color{black}with 
overwhelming probability}\footnote{We say that events occur ``with overwhelming probability'' (w.o.p.) if they occur with probability at least $1 - o(1/poly(n))$, where $o(1/poly(n))$ denotes the set of functions asymptotically smaller than the reciprocal of any polynomial function of $n$. Similarly, ``overwhelmingly small probability'' denotes any probability in the order of $o(1/poly(n))$.}.

The generalisation of the worst-case instance follows the construction in \cite{Witt2005} closely while adding an extra parameter to the problem in order to obtain a class of functions rather than a single instance for every possible value of $n$.
For the approximation results (especially for the hypermutation), we had to diverge significantly from Witt’s analysis, which relies on a constant upper bound on the probability that the assignment of a constant number of specific jobs to the machines will not  be changed for  at least a linear number of iterations. The hypermutation operator makes very large changes to the solution at every iteration and, unlike the standard bit mutation operator, we cannot rely on successive solutions resembling each other in terms of their bit-string representations. Thus, we have worked with the range of the objective function values. In particular, we divided the range of objective function values into $L$ intervals, where $L$ is the number of local optima with distinct makespan values. Measuring the progress in the optimisation task according to the interval of the current solution allows the achievement of the approximation result for the hypermutation operator and the bound on the  time until the ageing operator reinitialises the search process. The bound on the reinitialisation time in turn allows us to transform Witt’s bound on the probability that the standard \oneoneea finds the approximation in a single run into an upper bound on the approximation time for the $(\mu+1)$~EA$^{ageing}$.

\subsection{Mathematical Techniques for the Analysis}

{\color{black} The analysis of randomised search heuristics borrows most of its 
tools from probability theory and
stochastic processes
~\cite{feller1968}. We will frequently use the {\it waiting time 
argument} which states that the number of independent successive experiments, each with the same success probability $p$, that is necessary to observe the first success, has a geometric distribution with parameter $p$ and expectation $1/p$~\cite{mitzenmacher2005probability}. Another such 
basic argument that we will use is the {\it union bound} which bounds the probability of 
the union of any number of outcomes from above by the sum of their 
individual probabilities~\cite{mitzenmacher2005probability}. Another important class of results from  
probability theory which will be used are bounds on 
the probability that a random variable significantly deviates from its 
expectation (i.e., tail inequalities). The fundamental result in this class is  
{\it Markov's inequality} which states that the probability that a positive random variable $X$ is greater than any $k\in \mathbb{R}$ is at most the expectation of $X$ divided by $k$.
When the random variable of interest $X$ is either binomially or hypergeometrically distributed, the following stronger bounds, which are commonly referred as Chernoff bounds, are available:
\begin{theorem}[Theorem~10.1 and Theorem~10.5 in~\cite{DoerrNewTool}]\label{thm-cher}
If $X$ is a random variable with binomial or hypergeometric distribution, then, 
\begin{enumerate}
\item \emph{$\prob{X\leq(1-\delta)E[X]}\leq \exp{(-\delta^2 E[X]/2}$)} for any 
$\delta\in [0,1]$
\item \emph{$\prob{X\geq(1+\delta)E[X]}\leq 
\exp{(-\min{\{\delta, \delta^2\}}E[X]/3})$} for any $\delta\geq 0$
 %\left(\frac{e}{1+\delta}\right)^{(1+\delta) E[X]}$} for any $\delta>0$
\end{enumerate}
\end{theorem}
Note that the original statement of the theorem, like most Chernoff bounds do, 
refers to the sum of independent binary random variables, i.e., 
binomial distribution. However, it is shown in various works (including~\cite{DoerrNewTool} itself) that the theorem holds for the hypergeometric 
distribution as well.
We refer the reader to~\cite{DoerrNewTool, OlivetoHeYao2011, jansenbook, LehreOliveto2018, lenglerNew} for comprehensive surveys on the 
methods and approaches used for the analysis of randomised search 
heuristics. Finally, the following lower bound on the probability that a 
binomially distributed random variable deviates from its mean will also be used in our 
proof.
\begin{lemma}[Lemma~10.15 in \cite{DoerrNewTool}] \label{lem:binom} Let $X$ be a binomial random 
variable with parameters $n\geq0$ and $p=1/2$. Then, \emph{$\prob{X\geq
\frac{n}{2}+\frac{1}{2}\sqrt{\frac{n}{2}}} \geq \frac{1}{8}$}.
\end{lemma}
}

\subsection{General Preliminary Results}
%Having described the problem and the algorithms that will be analysed, 

{\color{black}In the rest of the paper we 
will refer to any solution $x$ to the \partition problem as a 
\emph{local optimum} if there exists no solution with smaller makespan which 
is at Hamming distance 1 from $x$. Moreover, $\mathcal{L}:=\{\ell_{1}, 
\ell_{2}, \ldots, \ell_{L}\}$ will denote the set of all local optima for a 
\partition instance where $L:=|\mathcal{L}|$ and the local optima $\ell_{i}$ 
are indexed according to non-increasing makespan.}

{\color{black}We can now state the following helper lemma {\color{black}which bounds the expected number of iterations that the \muoneeaageing and the \oneoneais with current fitness better than
$\ell_i$ require to reach a fitness at least as good as  $f(\ell_{i+1})$.}

{\color{black}
\begin{lemma}\label{lem:local}
 Let $x\in\{0,1\}^n$ be a non-locally optimal solution to the partition instance 
such that $f(\ell_i)>f(x)\geq f(\ell_{i+1})$ for some $i\in[L-1]$. 
% be the local optimal solution with 
% the largest makespan among all local optima whose makespans are at least as 
% good as $x$'s.
Then, 
\begin{enumerate}
\item The \oneoneais with current solution $x$ samples a solution $y$ such that $f(y)\leq f(\ell_{i+1})$ 
in  at most $2n^2$  expected iterations.
\item The \muoneeaageing with $\mu=O(\log{n})$, $\tau=\Omega(n^{1+c})$  for an 
arbitrarily small constant $c$ and current best solution $x$, 
samples a solution $y$ such that $f(y)\leq f(\ell_{i+1})$ in  $2en^2(1+o(1))$  
iterations  
with probability at least $1-n^{-\Omega(n^{c/2})}$.
\end{enumerate}

\end{lemma}
\begin{proof}
 Let $X_t$ be the best solution of the \muoneeaageing or the current solution
 of the \oneoneais at time $t$,
and $p_{t}^{*}$ denote the weight of the largest job in 
$X_t$ that can be moved from the fuller machine to the emptier machine such 
that the resulting solution yields a better makespan than $X_t$ (given that 
no other jobs are moved). 
Moreover, we will define $D_t:=f(X_t)-f(\ell_{i+1})$ and let $r_t$ be the 
smallest 
$j\in [n]$ such that $\sum_{i=n-j+1}^{n}p_i \geq D_t$.  

We will first prove that $p^{*}_{t}\geq p_{n-r_t+1}$ by contradiction. The 
assumption  $p^{*}_{t}<p_{n-r_t+1}$ implies that $\sum_{p_i\leq p^{*}_{t}}p_i< 
D_t$. Now, we consider a solution $y$ where all the jobs with 
weight at most $p_{t}^{*}$ in the fuller machine of $X_t$ are moved to the 
emptier machine but otherwise identical to $X_t$. Thus, 
$f(X_t)>f(y)>f(X_t)-D_t=f(\ell_{i+1})$ while $y$ is a local optimum. Since 
there exist no local optima with makespan value in $[f(x),f(\ell_{i+1})]$ by 
definition, the initial assumption creates a contradiction and we can conclude 
that $p^{*}_{t}\geq p_{n-r_t+1}$. 
% 

% For \oneoneais, the probability that in the first mutation step is at least $1/n$. Since the resulting solution has smaller makespan the hypermutation operator stops after this first mutation step. Due to the waiting time argument, in at most $n$ expected iterations 

We consider an iteration where only the job with processing time $p^{*}_{t}$ is moved from 
the fuller machine to the emptier machine of $X_t$, and will denote any such event as a \emph{large improvement}. After a large improvement, either the makespan will decrease by $p_{t}^{*}$ or the emptier machine in $X_t$ becomes the fuller machine in $X_{t+1}$ and $f(X_t)-p_{t}^{*}<f(X_{t+1})<f(X_t)$. We will denote the latter case as a \emph{switch}. 

In the former case, if the makespan decreases by $p_{t}^{*}$, then $r_{t+1}\leq r_t -1$ since $D_{t+1}=D_t - p_{t}^{*}\leq D_t - p_{n-r_t+1}$. Since $r_{t+1}\leq r_t\leq n$ by definition, after at most $n$ large improvements where $f(X_{t+1})=f(X_t)-p_{t}^{*}$, $f(X_t)\leq f(\ell_{i+1})$. 

If a switch occurs at time $t$, then the load of the emptier machine of $X_{t'}$ will be at least $f(X_t)-p_{t}^{*}$ for any $t'>t$ and moving a single job with processing time at least $p_{t}^{*}$ from the fuller machine to the emptier machine of $X_{t'}$ will yield a solution with makespan worse than $f(X_{t+1})$. Therefore, if a switch occurs at time $t$, by the definition of $p_{t}^{*}$, for any $t'>t$, $p_{t'}^{*}< p_{t}^{*}$. Since there are at most $n$ distinct processing times among all jobs, there can be at most $n$ large improvements throughout the optimisation process where a switch occurs. Thus, we can conclude that in at most $2n$ large improvements, $f(X_t)\leq f(\ell_{i+1})$.

We will now bound the waiting time until a large improvement occurs, $T_{imp}$, separately for the \oneoneais and the $(\mu+1)$~EA$^{ageing}$.

For the $(1+1)$~IA$^{hyp}$, the probability that the job with processing time $p^{*}_{t}$ is moved from 
the fuller machine to the emptier machine in the first mutation step is at least $1/n$. Since the resulting solution has smaller makespan, the hypermutation operator stops after this first mutation step. Due to the waiting time argument, $E[T_{imp}]=n$. Our first claim follows by multiplying this expectation with $2n$.

For the $(\mu+1)$~EA$^{ageing}$, we will first show that it takes $o(n)$ expected time until the best 
solution takes over the population if no improvements are 
found. Given that there are $k$ individuals in the population with fitness 
value $f(X_t)$, with probability $(k/\mu)(1-1/n)^{n}\geq (k/\mu)(1-1/n)e^{-1}$,
the \muoneeaageing picks one of the $k$ individuals as parent, does not flip any bit and adds a new solution with the same fitness value to the population. Thus, in 
expected $\sum_{k=1}^\mu(\mu/k)e(1-1/n)^{-1}\leq e \mu \ln {\mu}(1+o(1))=o(n)$ 
generations, the whole population consists of 
solutions with fitness value at least $f(X_t)$. 
In the following generations 
with probability at least $1/n (1-1/n)^{n-1}\geq 1/(en)$ the SBM only moves the 
job with 
weight $p_{t}^{*}$ and improves the best solution. Thus, in at most 
$e \mu \ln {\mu}(1+o(1))+en=en(1+o(1))$ expected  iterations of the \muoneeaageing a
large improvement occurs. 

Since $T_{imp}$ is a waiting time, we can use Markov's inequality 
iteratively in the following way to show that the ageing operator in the \muoneeaageing does 
not trigger with overwhelmingly high probability: 
% \begin{align*}
$\prob{T_{imp}\geq n^c \cdot E[T_{imp}]}\leq 
\prob{T_{imp}\geq 
n^{c/2}
\cdot E[T_{imp}]}^{n^{c/2}}=n^{-\Omega(n^{c/2})}$.
% \end{align*}
The second claim is obtained by multiplying the bound on the waiting time 
with $2n$.
% 
% % 
\end{proof}}
}

%The partition problem-AIS operators-Simple AIS algorithm
%The class of partition problems is a widely used class of combinatorial optimisation problems in the literature. Albeit it is a simple looking combinatorial problem, it is classified as an NP-hard problem, means no exact solution can be found for all possible instances of Partition problem in polynomial time. This class is generally defined as the problem of allocating $n$ jobs with different processing times to $m \geq 1$ machines. Different versions of the partition problem, also known as scheduling problem, are studied in the literature. Here, we consider a version in which is on allocation of $n$ jobs each with positive processing time $p_1 \cdots p_n$, to two identical machines such that the makespan (i.e., the load of the fuller machine) will be minimum \cite{NeumannWitt2010}. Any bit position in the bit string represents a job. Bit value 0 means it is allocated to machine 1 and bit value 0 means the corresponding job is on machine 2. The makespan cannot be more than the half of the total processing times of the jobs. 
%Hochbaum and 1+1 EA results. lemmas from Witt. 

%Notion and notations: the $n$ objects are sorted according to their weight $p_1 \geq p_2 \geq \cdots \p_n$ and the sum of the weights is $P=p_1+p_2+ \cdots p_n$. The fitness is the weight of the fuller machine and cannot be less than $P/2$. The sufficient conditions to improve a solution 

\begin{figure}[t] 
\centering
\includegraphics [scale=0.65]{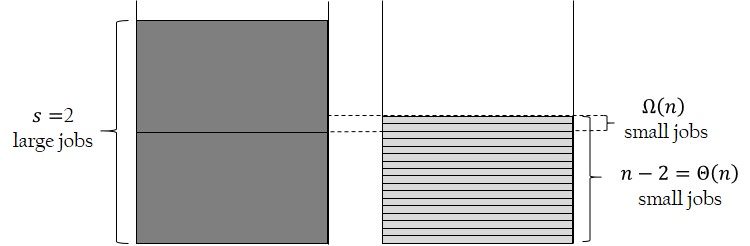}
\caption{The 4/3 approximation local optimum of $P^*_\epsilon$~\cite{Witt2005}}\label{fig:twojobs}
\end{figure}

\section{Generalised Worst-Case Instance} \label{sec:generalised}

%\begin{figure}[t]
% \centering
%\includegraphics [scale=0.65]{newinstanse}
%\caption{\WittGeneralisedInstance}
%\label{fig:generalised}
%\end{figure}

\begin{figure}[t]
 \centering
\includegraphics [scale=0.65]{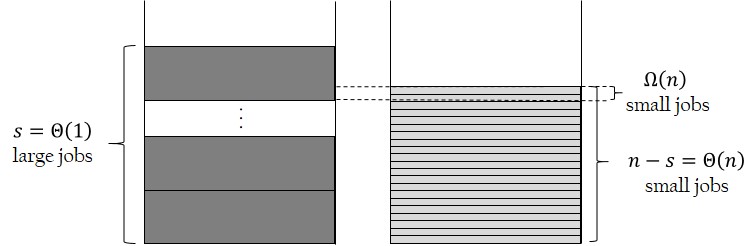}
\caption{A local optimum of \WittGeneralisedInstance}
\label{fig:generalised}
\end{figure}

The instance from the literature $P^*_\epsilon$ leading to a 4/3 worst-case expected approximation ratio for RLS and the \oneoneea consists of two large jobs with long processing times $p_1:=p_2:=1/3 -\epsilon/4$ and the remaining $n-2$ 
jobs with short processing times $p_i:=(1/3 + \epsilon/2)/(n-2)$ for $3\leq 
i\leq n$ (the sum of the processing times are normalised to 1 for cosmetic 
reasons)~\cite{Witt2005}.
 Any partition where one large job and half of the small jobs are placed on each machine is naturally a global optimum of the instance (i.e., the makespan is 1/2).
Note that the sum of all the processing times of the small jobs is slightly larger than the processing time of a large job.
The local optimum leading to the 4/3 expected approximation consists of the two 
large jobs on one machine and all the small jobs on the other. 
It is depicted in Figure \ref{fig:twojobs}.
The makespan of the local optimum is $p_1+p_2=2/3 -\epsilon/2$ and, in order to decrease it, a large job has to be moved to the fuller machine and at the same time at least $\Omega(n)$ small jobs have to be moved to the emptier machine. 
Since this requires to flip $\Omega(n)$ bits, which cannot happen with RLS and only happens with exponentially small probability $n^{-\Omega(n)}$ with the SBM of EAs,
their expected runtime to escape from such a configuration is at least exponential.

In this section, to highlight the power of the AIS to overcome hard 
local optima for EAs, we generalise the instance $P^*_\epsilon$  to contain an 
arbitrary number $s=\Theta(1)$ of large jobs and show how the considered AISs 
efficiently solve the instance class to optimality
%
%We will show that independently from the number of large jobs, the AIS can solve to optimality any instance of the generalised class 
by escaping from local optima efficiently 
while RLS and the \oneoneea  cannot. Hence, hypermutations and ageing are efficient at locating the optimal solution on a vast number of instances where SBM and local mutations alone fail.
The generalised instance class \WittGeneralisedInstance is defined as follows.
\begin{definition}\label{def-instance}
The class of \partition instances \WittGeneralisedInstance  are characterised by 
an even number of jobs $n$, an even number of large jobs $s=\Theta(1)$  and 
the following processing times: 
\begin{equation*}
p_i= 
\begin{cases}
      \frac{1}{2s-1}-  \frac{\epsilon}{2s}& \text{if } i\leq s, \\
      \frac{s-1}{n-s} \cdot (\frac{1}{2s-1}+ \frac{\epsilon}{2(s-1)}) & 
\text{otherwise,}
    \end{cases}
\end{equation*}
where $0<\epsilon < 1/(2s-1)$ is an arbitrarily small constant. 

\end{definition}
The instance class has the same property as the instance $P^*_\epsilon$. 
If all the large jobs are placed on one machine and all the small jobs on the other, then at least $\Omega(n)$ small jobs need to be moved in exchange for a large job to decrease the 
discrepancy (the difference between the loads of 
the machines). Such a local optimum allows to derive a lower bound on the worst-case expected approximation ratio of algorithms that get stuck there.
Obviously the greater the number of large jobs, the smaller the bound on the 
approximation ratio will be. 
{\color{black}Next, we will point out some important characteristics of the instance class. In particular, we will provide
upper bounds on the number of distinct makespan values that can be observed; first for all possible solutions and then only among local optima. 
Finally, we will provide restrictions on the distribution of the small jobs among the machines which hold for all suboptimal solutions.
\begin{property}\label{prop:levelwitt}
 Let $x\in\{0,1\}^n$ be a solution to the instance $I\in 
$ \WittGeneralisedInstance, $h$ and $k$ be the numbers of large and small jobs respectively 
on the fuller machine of $x$ and $\mathcal{L}\subseteq\{0,1\}^n$ be the set of local 
optima of $I$. Then,  

\begin{enumerate}
 \item $|\{y\in \mathbb{R}:\exists x\in\{0,1\}^n \;\text{s.t.}\; f(x)=y\}|=O(n)$
 \item $|\{y\in \mathbb{R}:\exists x\in \mathcal{L} \;\text{s.t.}\; f(x)=y\}|=O(1)$
\item $\forall x\in\{0,1\}^n $ s.t. $f(x)>1/2$, either $k=\Omega(n)$ and $x\notin \mathcal{L}$, or 
$\max(k,n-s-k)\geq \left( \frac{1}{2}+a\right)(n-s)$ for some $a=\Omega(1)$.
 \end{enumerate}
\end{property}
\begin{proof}
 Since $p_i=p_j=p_1$ for all $(i,j)\in [s]^2$ and $p_i=p_j=p_{s+1}$ for all 
$(i,j)\in ([n]\setminus[s])^2$, $f(x)= h \cdot p_1 +k\cdot p_{s+1}.$
 The first property follows from $(h,k)\in[s]\times[n-s]$ and $s=O(1)$.
 
 For the second property which considers only the local optima $\mathcal{L}$, 
note that if a local optimum $\ell$ with $h$ large jobs and 
$k$ small jobs exists, then a solution which has $h$ large 
jobs and $k'< k$ small jobs on its fuller machine cannot exist since, 
otherwise, moving a small job to the emptier machine would improve $\ell$.
Moreover, any solution with $h$ large jobs and $k'> k$ small 
jobs on its fuller machine cannot be a local optima since it can be improved by 
moving a small job from its fuller machine to its emptier machine. Thus, there 
exists a unique locally optimal makespan value for each $h\in[s]$. The property 
follows from $s=O(1)$.

In order to show the third property, we will partition the set of suboptimal solutions
$\mathcal{X}\subset \{0,1\}^n$ into three disjoint sets:

\begin{itemize}
 \item $\mathcal{X}_1:= \{x \in \X : f(x)\geq \sum_{i=1}^{s}p_i + 
\left(\frac{1}{s+1}\sum_{i=s+1}^{n}p_i \right)\}$,
 \item $\mathcal{X}_2:=\{ x\in \X\setminus \mathcal{X}_1 :h<s \wedge x 
\notin \mathcal{L}\}$ ,
\item $\mathcal{X}_3:=\X\setminus (\mathcal{X}_1 \cup 
\mathcal{X}_2)$,
\end{itemize}
for all $x\in\mathcal{X}_1$, since $ f(x)\geq \sum_{i=1}^{s}p_i + 
\left(\frac{1}{s+1}\sum_{i=s+1}^{n}p_i \right)>\frac{1}{2}+ 
\frac{1}{s+1}\sum_{i=s+1}^{n}p_i$, moving a single small job from the fuller machine
to the emptier machine always improves the 
makespan, thus $x\notin\mathcal{L}$. Moreover, since the sum of the processing times of the 
large jobs on the fuller machine 
can be at most $\sum_{i=1}^{s}p_i$, we get $k \geq (n-s)/(s+1) = \Omega(n)$. 

For any $x\in \mathcal{X}_2$, $x\notin \mathcal{L}$ by definition. We will prove by contradiction 
that $k=\Omega(n)$. Let us assume $k=o(n)$. Since the condition $h<s$ implies that there is at least
one large job in the emptier machine, the load of the emptier machine is at least
\begin{align*}
f(x)&\geq 
\left(\frac{1}{2s-1}-\frac{\epsilon}{2s}\right)
 +(n-s-o(n))\cdot 
\frac{s-1}{n-s}\left(\frac{1}{2s-1}+\frac{\epsilon}{2\left(s-1\right)}
\right) (1-o(1))  \\ &\geq \frac{1}{2} \left(1+\frac{1}{2 
s-1}+\frac{(s-1)\epsilon}{s}\right)(1-o(1)) > \frac{1}{2} \;\ \text{(for large enough $n$)},
\end{align*}
and contradicts that it is the emptier machine. Thus, $k=\Omega(n)$.

Finally, we will show that for all $x\in \mathcal{X}_3$, 
$\max(k,n-s-k)\geq \left( \frac{1}{2}+a\right)n$ for some $a=\Omega(1)$. When $x\in \mathcal{X}_3$ and $h=s$, then 
$k < \frac{n-s}{s+1}$ since $f(x)<\sum_{i=1}^{s}p_i + 
\left(\frac{1}{s+1}\sum_{i=s+1}^{n}p_i \right)$. The inequality $k < \frac{n-s}{s+1}$ in turn implies that $n-s-k>(n-s)-\frac{n-s}{s+1}=(n-s)(1-(1/(s+1)))\geq (1/2 + a) n$ for some $a$.
 If $x\in \mathcal{X}_3$ and $x\in \mathcal{L}$, then 
either there are no small jobs on the fuller machine ($k=0$), or the 
difference between the loads is so small that moving a small job to the emptier 
machine does not improve the makespan. For the latter case, since $f(x)>\frac{1}{2}$,
the machines cannot have the same number of small and large jobs. 
Moreover, since $s$ and $n$ are even intergers and $x\in\mathcal{L}$, 
if one of the machines has more than half of the large (small) 
jobs assigned to it, then more than half of the small (large) jobs has to be 
assigned to the other machine, because otherwise moving a single small job from the fuller machine
to the emptier one would improve the makespan.
 Without loss of generality (w.l.o.g.), let $M_1$ be the machine with more small jobs and fewer large jobs. 
Since $\sum_{i=1}^{n}p_i=1$, the load of $M_1$ is at least $1/2 - p_{s+1}/2$.
By definition, the number of large jobs in $M_1$, is bounded above by $s/2 -1$. 
Then, the contribution of small jobs to the load  of $M_1$ 
is at least \[
\frac{1}{2}\left(s\cdot p_1 + (n-s)p_{s+1} -p_{s+1}
\right)  - 
\left(\frac{s}{2}-1\right)p_1=\frac{
n-s-1}{2}p_{s+1} 
+ p_{1},
\]
where the first expression makes use of  $s\cdot p_1 + (n-s)p_{s+1}=1$. We 
can divide by $p_{s+1}$ and establish that the number of small jobs is at least:
\begin{align*}
 \frac{n-s-1}{2} +\frac{p_{1}}{p_{s+1}} >\frac{n-s-1}{2}  + \frac{n-s}{s} 
 =(n-s)\left(\frac{1}{2}+\frac{1}{s}\right) \left(1-O\left(\frac{1}{n}\right)\right). 
\end{align*}
Thus, for all $x\in \mathcal{X}_3$, $\min{(k, n-s-k)}\geq 
(n-s)(\frac{1}{2}+a)$ where 
$a:=\min{(\frac{1}{s}-O(\frac{1}{n}),\frac{1}{2}-\frac{1}{s+1})}=\Omega(1)$.
\end{proof}}
%Such ratios are derived in the following lemma.
%\begin{lemma}
%For the \WittGeneralisedInstance class of instances, any solution which assigns all of the jobs of size $\frac{1}{2s-1}-  
%\frac{\epsilon}{2s}$ to a single machine cannot have an 
%approximation ratio smaller than $\frac{1}{2 s-1}-\epsilon +1$.
%\end{lemma}
%%
%%
%\begin{proof}
%If a solution $x$ assigns all of the large jobs (i.e., jobs of processing time 
%$\frac{1}{2s-1}-  \frac{\epsilon}{2s}$ ) to a single machine, then its load (i.e., total processing time) is at 
%least:
%\[
 %f(x) \geq  s \left(\frac{1}{2 s-1}-\frac{\epsilon}{2 s}\right)
%\]
%Since the makespan of the optimal solution is $1/2$, the approximation ratio of 
%$x$ is at least: 
%\[
% 2 s \left(\frac{1}{2 s-1}-\frac{\epsilon}{2 s}\right)=\frac{1}{2 s-1}-\epsilon 
%+1
%\]
%[Donya: If we already know the opt is 1/2, why do we need this Lemma?]
%\end{proof}
%
We now prove that the \oneoneea has exponential expected runtime on \WittGeneralisedInstance. 
The proof is similar to that of the 4/3 approximation~\cite{Witt2005}. It essentially shows that with constant probability the {\color{black}algorithm gets trapped in a state where it is exponentially unlikely to move the extra large jobs out of its fuller machine.}
That RLS also fails is essentially a corollary.
\begin{restatable}{theorem}{theoremone}\label{th:1+1}
The \oneoneea needs at least $n^{\Omega(n)}$ fitness function evaluations in expectation to optimise any instance of \WittGeneralisedInstance. 
\end{restatable}

\begin{proof}
The proof is similar to that of the 4/3 approximation~\cite{Witt2005}.
We will first prove that with constant probability the \oneoneea 
initialises with all large jobs in $M_1$ and then fills $M_2$ 
with a $(1-\frac{ \epsilon}{2(s -1)})$ fraction of the small jobs before moving 
any large jobs to $M_2$ (Event $\mathcal{E}_1$).
From then on, we will show that the algorithm requires exponential expected 
time to move a large job. Since the sum of the processing times of large jobs 
is larger than $1/2$ (the makespan of the optimum), the expected 
optimisation time will be exponential as well.

With probability $(1/2)^{s}=\Theta(1)$, all large jobs are assigned to the same 
machine ($M_1$ w.l.o.g.) in the initial solution. 
%Hence the makespan S is at least $s/(2s-1) - \epsilon/2$.
The expected number of small jobs initially assigned to 
$M_1$ is $(n-s)/2$ and since they are assigned independently, the 
probability that more than $(n/2) + n^{3/5}$ of them are assigned is 
exponentially small due to Chernoff bounds. 

Let $k_t$ be the number of small jobs on the fuller machine $M_1$
at time $t$. As long as the large jobs are not moved, new solutions are only 
accepted if the number of small jobs in $M_2$ increases or stays the same. 
%Conditional on large jobs never having moved, 
A single small job is moved to $M_2$ with probability 
$(k_t/n)(1-1/n)^{n-s-1}\geq k_t /(e n)$.  
% at iteration $t \leq \frac{2 e (s-1)^2 }{s} $. 
 
 The expected time conditional on large jobs never having been moved until a 
$(1-\frac{ \epsilon}{2(s -1)})$ fraction of the small jobs are moved to $M_2$ 
is:
\begin{align*}
\sum_{k=\epsilon(n-s)/(2(s-1))}^{n/2+n^{3/5}} \frac{e n}{k} 
	&\leq  
	e n \sum_{k=\epsilon(n-s)/(2(s-1))}^{n/2 +n^{3/5}} 
\frac{2(s-1)}{\epsilon\left(n-s\right)}
	\\&\leq
	2 (s-1) \left(\frac{e}{\epsilon}\right)\left(\frac{n}{n-s}\right) \left(\frac{n}{2}+n^{3/5}\right) 
\\
%&= 2 (s-1) \frac{e}{\epsilon}\frac{\left(1-\epsilon\right)\left(1 + 
% o\left(1\right)\right) n }{2 s} \left(1+o(1)\right)\\
&\leq 2 (s-1) \frac{e}{\epsilon}\frac{n}{2} \left(1+o(1)\right) \leq  O(n)
%= \left(1+o(1)\right)\frac{e (s-1)^2 }{s} n
\end{align*}
Let $c$ be a constant such that $cn-1$ is at least twice as large as the above 
expectation. Due to Markov's inequality at least a $(1-\frac{ 
\epsilon}{2(s -1)})$ fraction of the small jobs are moved to $M_2$  in $cn-1$ 
iterations with probability at least $1/2$.  Now, we calculate the probability 
that no large jobs have been moved in $cn-1$ iterations. The probability that 
SBM does not move any large jobs to the emptier machine ($M_2$) is $1-s/n$ in 
each iteration. The probability that the large jobs are not moved during the 
first $c n -1$ iterations is at least  $(1-s/n)^{cn-1} \geq  e^{- c s 
}=\Omega(1)$. Hence,  Pr$\{\mathcal{E}_1\}=\Omega(1)$.

Now that we found the probability of event $\mathcal{E}_1$, we can find the 
expected runtime  conditional on $\mathcal{E}_1$.

Let at least 
% $k':= n-s - (\epsilon/2)(n-s)/(s-1) = (n-s) \big[(1-\frac{ 
% \epsilon}{2(s -1)})\big]$ small jobs (i.e., 
a $(1-\frac{ \epsilon}{2(s -1)})$ fraction of small jobs and all the large 
jobs be on $M_2$. 
The sum of the processing times of small jobs on $M_2$ is greater than 
the processing time of ($s-1$) large jobs (i.e., $\frac{s-1}{2s-1} - 
\frac{\epsilon}{2}\frac{s-1}{s}$) and greater than

\begin{align*}
 & (s-1) \left(\frac{1}{2s-1}+\frac{\epsilon}{2(s-1)}\right) - 
\frac{\epsilon}{2}\left(\frac{1}{2s-1} + \frac{\epsilon}{2(s-1)}\right)
\\ & = \frac{s-1}{2s-1} + \frac{\epsilon}{2} \bigg(1 - \frac{1}{2s-1} - 
\frac{\epsilon}{2(s-1)}\bigg)
{\color{black}> \frac{s-1}{2s-1}+ \frac{\epsilon}{3} - \frac{\epsilon^2}{4} > 
\frac{s-1}{2s-1} + \frac{\epsilon}{4}.}
\end{align*}

Then the makespan can be at most $S< 1 - \big(\frac{s-1}{2s-1} + 
\frac{\epsilon}{4}\big) = \frac{s}{2s-1} - \frac{\epsilon}{4}$. If a large job 
was to be transferred to $M_2$, then its total time would become $S' > 
\frac{s}{2s-1} + \frac{\epsilon}{4} - \frac{\epsilon}{2s}>\frac{s}{2s-1}$. 
Hence, in order for a large job to be transferred to $M_2$, a processing time of 
at least $\epsilon/4$ has to be transferred to $M_1$ in the same iteration. Then 
the conditional expected time to transfer a large job would be exponential since 
at least a linear number of small jobs need to be moved from $M_2$ to $M_1$ such 
that their processing times of order $O(1/n)$ sum up to a constant $\epsilon/4$. 
{\color{black}Such a mutation requires at least a linear number of bits to be 
flipped  which happens with probability at most $n^{-\Omega(n)}$. The 
probability that such a mutation happens in $n^{\delta n}$ steps is still in the 
same asymptotic order for a sufficiently small constant $\delta$. Since any mutation
that increases the load of $M_2$ will not be accepted and the large jobs must be equally 
distributed among the machines in the optimal solution, with $\prob{\mathcal{E}_1}\cdot (1-n^{-\Omega(n)}) =\Omega(1)$ the runtime will be in the order of $n^{\Omega(n)}$. 
The theorem statement follows since $E(T) \geq E(T \mid \mathcal{E}_1)\cdot \text{Pr}\{\mathcal{E}_1\}=n^{\Omega(n)}$ by the definition of expectation of a random variable. 
}
\end{proof}

In the following subsection, we will show that hypermutations are efficient. 
Afterwards we will do the same for ageing. 
\subsection{Hypermutations}

We will start this subsection by proving some general statements about 
hypermutations. The hypermutation operator flips all the bits in a bitstring 
successively and evaluates the fitness after each step. Unless an improvement is 
found first, each hypermutation operation results in a sequence of $n$ solutions 
sampled in the search space. This sequence of solutions contains exactly one 
bitstring for each Hamming distance $d\in [n]$ to the initial bitstring. 
Moreover, the string which will be sampled at distance $d$ is uniformly 
distributed among all the bitstrings of distance $d$ to the input solution. 

{\color{black}The number of possible outcomes grows exponentially in the number of flipped bits until the ($n/2$)th bit is flipped and afterwards reduces at the same rate.
Thus, the first and the last constant number of sampled  strings 
are picked among a polynomially large subset of the search space.}
We will now provide a 
lemma regarding the probability of particular outcomes in the first and 
the last $m$ bitstrings of the sequence.
% \begin{restatable}{lemma}{bitflip} 
% \label{lem:bitflip}
% Given that no improvement is found, the probability that $k$ specific bit 
% positions are flipped by the hypermutation operator  in the first (or the last) 
% $m\geq k$ mutation steps is at least 
% % $\prod_{j=0}^{k-1}{\frac{m-j}{n-j}}\geq 
% $\left(\frac{m-k+1}{n-k+1}\right)^{k}$.
% \end{restatable}
% 
% \begin{proof}
% Picking a uniformly random permutation in which to flip the bits is equivalent to the random process where a number in $[n]$ is assigned to each bit position without replacement. 
% W.l.o.g., we index the $k$ bit positions  of interest from $x_1$ to $x_k$. The probability that bit $x_1$ is assigned a number in $[m]$ is $m/n$. Given that $\ell$ numbers in $[m]$ have been assigned to bits $x_i$ for all $i$ in $[\ell]$,  the probability that the $(\ell +1)$th bit to be assigned a number from $[m]$ is at least $(m-\ell)/(n-\ell)$. Therefore, the 
%  probability that $k$ specific bit mutations occur in the first $m$ 
% mutation steps is at least $\prod_{j=0}^{k-1}\frac{m-j}{n-j}$. Since $n\geq 
% m$, 
% $\frac{m-j}{n-j}\geq \frac{m-k}{n-k}$, for all $j \in \{0,1,\ldots,k-1\}$, and 
% consequently the series product is bounded from below by 
% $\left(\frac{m-k+1}{n-k+1}\right)^{k}$.
% 
% \end{proof}
% 
%                                           
\begin{restatable}{lemma}{hypinterv}\label{lem:hypinterv}
{\color{black}Let $x^{i} \in \{0,1 \}^{n}$ be the $i_{th}$ bitstring sampled by 
the hypermutation and for any arbitrary subset of indices $S\subset [n]$, let  
$x^{i}_{S}$  be the subsequence of $x^i$ which consists of bit positions in 
$S$. For 
any given target string $s^*\in\{0,1 \}^{|S|}$  and } integer 
$m>|S|$, the probability that $\forall i \in \{m,  \ldots, n-m 
\}$,   $x^{i}_{S}= s^*$ is at least 
$\left(\frac{m-|S|+1}{n-|S|+1}\right)^{|S|}$. 
\end{restatable}

\begin{proof}
Let $S_A$ and $S_D$ denote the set of positions where the initial subsequence 
$x^{0}_{S}$ and $s^*$ agree and disagree respectively and $ord(j)\in[n]$ 
denote the 
mutation step when the bit position $j\in S$ is flipped. Then,
$\forall i \in \{m,  \ldots,  n-m \}, \enskip  x^{i}_{S}= s^*$ holds if and 
only if $\forall j\in S_D,\enskip ord(j)\leq m$ and $\forall j\in S_A, 
\enskip ord(j)\geq n-m$. Given that there are at least $k$ different $j\in 
S_D$ such that $\enskip ord(j)\leq m$, the probability that there are at least 
$k+1$  different $j\in S_D$ such that $\enskip ord(j)\leq m$ is 
$(m-k)/(n-k)$. Thus, the probability that  $\enskip ord(j)\leq m$ 
holds for all $j\in S_D$ is at least $\prod_{k=0}^{|S_D|-1}\frac{m-k}{n-k}$. 
Similarly, if $\enskip ord(j)\leq m$ holds for all $j\in S_D$ and there are at 
least $k$ different $j\in S_A$ such that $\enskip ord(j)\geq n-m$, the 
probability that  there are at least $k+1$  different $j\in S_A$ such that 
$\enskip ord(j)\geq n-m$ is at least 
$(m-k)/(n-k-|S_D|)$. Thus, the probability that, $\forall j\in S_D,\enskip 
ord(j)\leq m$ and $\forall j\in S_A, \enskip ord(j)\geq n-m$ is at least 
$\left(\prod_{k=0}^{|S_D|-1}\frac{m-k}{n-k}\right)\left(\prod_{k=0}^{|S_A|-1}
\frac { m-k } { n-k-|S_D| } \right)\geq 
\left(\frac{m-|S|+1}{n-|S|+1}\right)^{|S|}$.
\end{proof}
We can observe that the probability bounds we get from this lemma are polynomial if $|S|$ is a constant. Moreover, if both $|S|=\Theta(1)$ and $m=\Omega(n)$ we obtain probabilities in the
order of $\Omega(1)$. 
% As an analogy, the more time you allow your actors to get 
% dressed/undressed before/after the show,  more elaborate the costumes that they 
% can pull off during the show. The caveat is that the total time is fixed and the 
% "show" has to be shorter. 
% {\color{red}{I would remove the analogies as the reviewers are unlikely to be 
% more familiar with theatre than with probabilities :) Nevertheless you guys 
% decide.}}

For \WittGeneralisedInstance, we would like to distribute two kinds of jobs 
evenly among the machines. While one type of job is constant in number, the other is 
in the order of $\Theta(n)$. So the previous lemma would provide reasonable 
probability bounds only for the large jobs. For the small jobs we will make use 
of the fact that the configuration we want is not any configuration, but an exact 
half and half split. Intuitively, it is obvious that if all the jobs start on one 
machine, at exactly the $\frac{n}{2}$th bit flip the split will be half and half. 
However, as the starting distribution gets further away from the extremes, it 
becomes less clear when the split will exactly happen. 
% Continuing the analogy, you would prefer the audience to be punctual if the show 
% is elaborate.
%
Fortunately, the fact that the number of small jobs is large will work in our 
favor to predict the time of the split more precisely.

\begin{restatable}{theorem}{hypmain}\label{thm:hypmain}
If the input bitstring of hypermutation has $\left(\frac{1}{2} + a\right)n$ 
1-bits for 
some constant $a>0$, then with  probability $1-e^{-\Omega(n c^2)}$, 
{\color{black}for any 
$c=\omega(\sqrt{\log{n}/n})$ such that $c\in (0,a)$, there 
exists a $k\in\{n\frac{a-c}{2a-c},\ldots,n\frac{a}{2a-c}\}$,} such that the 
number of 1-bits in the 
$k_{th}$ solution sampled by hypermutation has exactly $n/2$ 1-bits.
% {\color{yellow} If the input bitstring of hypermutation has $\left(\frac{1}{2} + a\right)n$ 
% 1-bits for some constant $a>0$, then the probability that any  solution sampled 
% after the $n\frac{a}{2a-c}$th  mutation  step for any  {\color{black}$c\in(0,a)$} 
% and $c=\omega(\sqrt{\log{n}/n})$ to have 
% more than $n/2$ 1-bits is in the order of $e^{-\Omega(n c^2)}$.}
\end{restatable}
%and $X_k$ the last sampled bit-string
\begin{proof}
 Let $Y_k$ be the total number of 1-bits flipped by the hypermutation operator up to and including the $k$th mutation step.
Given any parent string $x\in\{0,1\}^n$ with $\left(\left(1/2\right)+a\right)n$ $1$-bits, $Y_k$ is a hypergeometrically distributed random variable with parameters $n$, $\left(\left(1/2\right)+a\right)n$, $k$ and expectation $E[Y_k]=\left(\left(1/2\right)+a\right)\cdot k$.

 Let $X_k$ be the bit-string sampled after the $k$th mutation step. It has $\left(\left(1/2\right)+a\right)n-Y_k$ of $x$'s $1$-bits which have never been flipped and $k-Y_k$ of its $0$-bits which are flipped to $1$-bits. Thus, the number of 1-bits in $X_k$, $|X_k|_1=\left(\left(1/2\right)+a\right)n+k-2Y_k$. Now, we can bound the probability that $|X_k|_1\geq n/2$ after a particular mutation step $k\geq \frac{n\cdot a}{2a -c}$ by using Theorem~\ref{thm-cher}. 
%%be the random bit-string sampled at the $k$th mutation step of the hypermutation operator. Since $|x|_1=\left(\left(1/2\right)+a\right)n$ is assumed, 
%$$|X_k|_1=(|x|_1 - Y_k )+ (k - Y_k)=\left(\left(1/2\right)+a\right)n+k-2Y_k,$$
 %where $Y_k$ is the number of 1-bits flipped after the $k$th mutation step,  
\begin{align*}
\prob{ |X_k|_1>n/2}&=\prob{Y_k < \frac{k + a \cdot n}{2} } 
= \prob{Y_k < \frac{k + a \cdot n}{k + 2 a \cdot k} \cdot E[Y_k]  } 
\\ &\leq \prob{Y_k < \left(1- \frac{c}{2a+1}\right)\cdot E[Y_k]  } 
\leq \exp{\left(-\left(\frac{c}{1+2a}\right)^{2} \frac{E[Y_k]}{2}\right)}\\
&\leq 
\exp{\left(-\left(\frac{c}{1+2a}\right)^2\frac{\left(\frac{1}{2}+a\right)\frac{
n\cdot a}{2a-c}}{2}\right)}= e^{-\Omega(n c^2)},
\end{align*}
where we used $k\geq \frac{n\cdot a}{2a-c}$ to obtain the first and the 
last inequality and Theorem~\ref{thm-cher} to obtain the second. Then, using 
the same line of arguments, we bound $\prob{ |X_k|_1<n/2}$ at  any given 
mutation step $k\leq n\frac{n \cdot (a-c)}{2a-c}$ by $e^{-\Omega(n c^2)}$:
\begin{align*}
\prob{ |X_k|_1<n/2}&=\prob{Y_k > \left(1+ \frac{a\cdot 
c}{(2a+1)(a-c)}\right)\cdot E[Y_k]  } \\
&\leq  \exp{\left(-\left(\frac{a\cdot c}{(2a+1)(a-c)}\right)^{2} 
\frac{E[Y_k]}{3}\right)}\leq e^{-\Omega(n c^2)},
\end{align*}
where in the last inequality we used that $a/(2a+1)=\Omega(1)$ and that $(a-c)=o(1)$ only when $c=\Omega(1)$.
Finally, we will use the union bound and the $c=\omega\left(\sqrt{\frac{\log{n}}{n}}\right)$ assumption to prove our claim.
\begin{align*}
&\prob{\forall k \in \left[\frac{n \cdot (a-c)}{2a-c},\frac{n \cdot 
a}{2a-c}\right]: |X_k|_1 \neq n/2} \\ &\leq 
\sum\limits^{n}_{j=\frac{n\cdot a}{2a-c}}\prob{|X_j|_1>n/2} + 
\sum\limits^{\frac{n\cdot (a-c)}{2a-c}}_{j=1}\prob{|X_j|_1<n/2}\leq n\cdot 
e^{-\Omega(n c^2)}= e^{-\Omega(n c^2)}.
\end{align*}
\end{proof}
Now we have all the necessary tools to prove that the \oneoneais can 
 solve \WittGeneralisedInstance efficiently. {\color{black}To prove our result, 
we will use the combination of Property~\ref{prop:levelwitt}.3, Lemma~\ref{lem:hypinterv}, and Theorem~\ref{thm:hypmain} to show that the probability of improvement is at least $\Omega(1)$ for any suboptimal solution and then obtain our bound on the expected runtime by Property~\ref{prop:levelwitt}.1, which indicates that there can be at most $O(n)$ improvements before the optimum is found.
The heart of the proof of the following theorem is to show that from any local 
optimum, hypermutations identify the global optimum with constant probability 
unless an improvement is found first. }

\begin{restatable}{theorem}{thaislocal}
The \oneoneais optimises the \WittGeneralisedInstance class of instances  in $O(n^2)$ expected fitness function evaluations.
\end{restatable}

\begin{proof}
According to Property~\ref{prop:levelwitt}.3, for any suboptimal solution $x$, either
there are at least a linear number of small jobs that can be moved from the fuller machine to the emptier machine to improve the makespan or the number of small jobs in each machine differs
from $\frac{n-s}{2}$ by a linear factor. 

For the former case, in the first mutation step the probability that hypermutation flips one of the bits which correspond to the small jobs in the fuller machine is at least $\Omega(n)/n=\Omega(1)$. Since such a mutation creates an improved solution, the hypermutation operator stops and the overall improvement probability is at least $\Omega(1)$.

For the latter case, the number of small jobs in one of the machines
is at least $\left(\frac{1}{2}+a\right)(n-s)$ for some positive constant $a$.
Using the notation of Theorem~\ref{thm:hypmain} with $a$ and $c=1/(a\log{n})$, 
with overwhelming 
probability the number of small jobs assigned to each machine will be exactly 
$(n-s)/2$ at some bit-flip between $(n-s) 
\frac{(a)-1/(a\log{n})}{(2/a)-1/(a\log{n})}=(n-s) 
\frac{1-(1/\log{n})}{2-1/\log{n}}$ and 
$(n-s)\frac{a}{(2 a)-1/(a\log{n})}=(n-s)\frac{1}{2-1/(\log{n})}$.  According to Lemma~\ref{lem:hypinterv}, the 
probability that the large jobs are distributed evenly throughout this
interval is at least:
\[
 \left(\frac{(n-s)\frac{1-(1/\log{n})}{2-(1/\log{n})}-s +1}{n- s + 1}\right)^{s}=2^{-s}
(1-o(1))=\Omega(1).
\]
%The machine with 
%more large jobs is $M2$. Assuming that there are $(s/2)+k$ large jobs on $M2$ 
%for some $k<s/2$, we will lower bound the probability that $k$ specific large 
%jobs are moved from $M2$ before the interval while $s-k$ remaining jobs stay put 
%until the interval ends. 
%Since, that a large job is moved between the $i$th and $i+1$th small jobs does 
%not increase or decrease the probability that another large job is also moved 
%between the $i$th and $i+1$th small jobs, the mutation step that any large job 
%is moved is independently and uniformly  distributed in between the $(n-s)$ 
%steps where the small jobs are moved. Thus, considering that the interval 
%stretches symmetrically around $n/2$ the probability that $k$ specific large 
%jobs are moved from $M2$ before the interval and $s-k$ remaining jobs stay put 
%until the interval ends is:
%\[
 %\left((n-s)\frac{1-(1/\log{n})}{2-1/\log{n}}\frac{1}{n-s}\right)^{s}=2^{-s}
%(1-o(1))=\Omega(1)
%\]
Therefore, with probability $\Omega(1)$ either the optimal solution will be 
found or the hypermutation will stop due to an improvement.   
Since according to Property~\ref{prop:levelwitt}.1 there are at most $O(n)$ 
different makespan values and the probability of improvement is at least $\Omega(1)$ for all sub-optimal solutions, the optimum is found in expected $O(n)$ iterations. Given that the hypermutation operator evaluates at most $n$ solutions in each iteration, our claim follows.
\end{proof}

Since the worst case instance for the \oneoneea presented in~\cite{Witt2005} is an instance 
of \WittGeneralisedInstance with $s=2$, the following corollary holds.
\begin{corollary}
The \oneoneais optimises $P_{\epsilon}^*$ in $O(n^2)$ expected fitness function evaluations. 
\end{corollary}

\subsection{Ageing}

In this section we will show that the \muoneeaageing  can optimise the
\WittGeneralisedInstance instances efficiently. Our approach to prove the following theorem is to first show 
that if a solution where the large jobs are equally distributed is sampled in 
the initialisation, then it takes over the population and the \muoneeaageing 
quickly finds the optimum. The contribution of ageing is considered afterwards 
to handle the case when no such initial solution is sampled. Then we will show 
that whenever the algorithm gets stuck at a local optima, it will reinitialise 
the whole population after $\tau$ iterations and sample a new population. 
Before we move to the main results, we will prove two helper lemmata. In the 
first lemma, we will closely follow the arguments of Lemma~5 in~\cite{OlivetoSudholt2014}
to get a similar but more precise bound.

\begin{lemma}\label{lem:ageto}
 Consider a population of the \muoneeaageing which consists of solutions of equal 
fitness. Given that no improvement is found and no solutions reach age $\tau$ 
before, the expected time until the population consists of solutions with the 
same age is at most $144\mu^3$.
\end{lemma}
\begin{proof}
 Let $X_t$ denote the largest number of 
individuals in the population having the same age. While $X_t<\mu$, the number 
of individuals increases if one of the $X_t$ individuals is selected as parent, 
the SBM does not flip any bits and a solution from another age group is removed 
from the population.
\[\prob{X_{t+1}=X_t +1\mid X_t}\geq \frac{X_t(\mu-X_t)}{\mu (\mu+1)}\cdot 
\left(1-\frac{1}{n}\right)^{n}\]
Similarly, $X_t$ decreases only if a solution from another age group is cloned 
and one of the $X_t$ solutions are removed from the population.
\[\prob{X_{t+1}=X_t -1\mid X_t}\leq\frac{(\mu-X_t)X_t}{\mu (\mu+1)}\cdot 
\left(1-\frac{1}{n}\right)^{n} \leq \prob{X_{t+1}=X_t +1|X_t}\]

Let a \emph{relevant step} be any iteration where $X_t \neq 
X_{t+1}$. Then, the number of generations where $X_{t+1} \geq X_{t}$ in 
$2\mu^2$ relevant steps has a distribution that stochastically dominates the 
binomial distribution with parameters $2\mu^2$ and $1/2$.
Since for any random variable $X\sim Bin(2\mu^2,1/2)$, $\prob{X\geq 
\mu^2+\frac{1}{2}\sqrt{\mu^2}}\geq \frac{1}{8}$ (Lemma~\ref{lem:binom}), the 
probability that 
$X_t$ increases at least $\mu^2+\mu/2$ times in 
a phase of $2\mu^2$ relevant steps is at least $1/8$. Since $X_t\geq 1$, 
the expected number of relevant steps until $X_t=\mu$ is at most 
$(1/8)^{-1}2\mu^2=16\mu^2$. The probability of a relevant step is at least
\[\frac{(\mu-X_t)X_t}{\mu (\mu+1)}\cdot 
\left(1-\frac{1}{n}\right)^{n} \geq \frac{(\mu-1)}{\mu (\mu+1)}\cdot 
\left(1-\frac{1}{n}\right)^{n}\geq \frac{1}{3\mu}\cdot 
\frac{1}{3}=\frac{1}{9\mu}, \]
for any $\mu\geq 2$ and $n\geq 3$. Thus, the expected waiting time for a 
relevant step is at most $9 \mu$ and the total expected time until the 
population reaches the same age is at most $144 \mu^3$.
\end{proof}

\begin{lemma}\label{lem:restart}
 For any partition instance with $L$ local optima with different makespans, the 
expected time until 
the \muoneeaageing with $\mu=O(\log{n})$ and $\tau=O(n^{1+c})$ for any constant~$c$ either reinitialises the population or finds the optimum is at most
$L(2en^2+\tau)(1+o(1))$.
\end{lemma}

\begin{proof}
%  For $\ell_i\in\{0,1\}^n$, let $\mathcal{L}:=\{\ell_1, \ell_2, \ldots, 
% \ell_L\}$  denote the set of local optima.  W.l.o.g., we index $\ell_i$
% according to decreasing makespan, \emph{i.e.}, $f(\ell_1)\geq f(\ell_2) \geq 
% \cdots \geq f(\ell_L)$. 
Let $X_t$ denote the solution with the smallest makespan in the population at 
time~$t$.  
Starting from any current population, consider a phase of length 
$c'L(n^2+\tau)$ for some constant $c'>3e$. 
We will make a case distinction with respect to whether the following 
statement holds during the entirety of the phase.

 $$\forall t\enskip X_t\in\mathcal{L}:  \exists t'\leq 
t+\tau+145\mu^4 \text{ s.t. }f(X_{t'})<f(X_t)$$ 
The statement indicates that whenever a local optimum becomes the best solution 
in the population, a solution with a better makespan is found in at most 
$\tau+145\mu^4$ iterations.

Case 1: If the statement holds throughout the phase then, at 
most $L(\tau+144\mu^4)=O(L\cdot \tau)$ iterations satisfy $X_t=\ell_i$ for some 
$i$. Since the \muoneeaageing 
% initialised with any population 
with $f(\ell_i)> f(X_t) \geq 
f(\ell_{i+1})$ samples a solution $x$ with $f(x)\leq f(\ell_{i+1})$ in 
$2en^2(1+o(1))$ 
iterations  with probability at least $1- n^{-\Omega(n^c)}$   according to 
Lemma~\ref{lem:local}, by the union bound the total number of iterations 
during the phase where 
the current best solution is not a local optimum is at most $L\cdot 2en^2 \cdot 
(1+o(1))$ 
with probability at least $1- n^{-\Omega(n^c)}\cdot 2^n =1-o(1)$. Therefore, in $L\cdot 3en^2$ iterations the optimal solution is found 
with at least constant probability.
% (due to union bound and $L\geq 2^n$)

Case 2: If the statement does not hold for some $t$, then 
there exists a time interval of length $\tau + 145\mu^4$ during the phase where 
no 
improvements will occur and the current best individual is a local optimum. 
Given that there are $k$ individuals which share the best 
fitness value, with probability $(k/\mu)(1-1/n)^{n}\geq (k/\mu)(1-1/n)e^{-1}$, 
SBM creates a copy of one of these solutions. Since the fitness cannot be 
improved by our assumption, in at most $\sum_{k=1}^{\mu}(\mu/k)(1-1/n)^{-1}e 
=(1+o(1))e \mu \ln{\mu}$ expected time the population only consists of 
individuals with the best fitness. This occurs in less than $2 e \mu 
\ln{\mu}$ iterations with probability at least $1/3$ due to Markov's 
inequality. 
According to Lemma~\ref{lem:ageto}, unless the best fitness is 
improved, in at most $144\mu^3$ steps 
the population reaches the same age. Using Markov's 
inequality, we can bound the probability that 
this event happens in at most $145\mu^3$ iterations by $\Omega(1)$ from 
below. 
Thus,  given that the best fitness does not improve in $\tau + 145\mu^4$ 
iterations, in $2e\mu\ln{\mu}+145\mu^3$ iterations, the 
population consists of solutions with the same fitness and the same age with 
constant probability, which leads to a reinitialisation before the phase ends. 

If neither the optimal solution is found nor the population is reinitialised, we 
repeat the same argument starting from an arbitrary population. Since, the 
probability of at least one of the events happenning is constant, in a constant expected 
number of phases of length $L(\tau+145\mu^4)=L\cdot \tau (1+o(1))$, 
one of the events occurs. 
% 
% $\forall X_{t}=\ell_i$ $i\in[L]$ 
% 
% for 
% any $t$  $X_t=X_{t+1}=\cdots=X_{t+\tau+\mu^4}=\ell_i$ for some $i\in[L]$ during 
% this phase, 
%    
%    if the time spent two 
% improvements to the best fitness in the population is never larger than $\tau + 
% \mu^4$ until 
% the optimum is found, then the runtime is in the order of $O(n \tau)$ since 
% there are $O(n)$ different fitness levels and $\mu=O(\log{n})$. We will 
% now consider a phase of length $\tau + \mu^4$ after an improvement where no 
% other improvements will occur. 
% 
\end{proof}
{\color{black} 
% In order to prove the upper bound on the expected runtime, we 
% will enumerate a sequence of events $\mathcal{E}_i$ for $i\in[5]$, 
% such that $p_{success}:=\bigcap\limits_{i=1}^{5}\mathcal{E}_i=\Omega(1)$, 
% that 
We will now show that with probability at least $p_{success}=\Omega(1)$, the 
optimal solution will be found in $O(n^{3/5})$ iterations starting from a 
randomly initialised population.  
Since Lemma~\ref{lem:restart} and Property~\ref{prop:levelwitt}.2 imply that the 
expected time until either the optimum is found or the population is 
reinitialised is in the order of $O(n^2+\tau)$, in 
$O(n^{3/5})+O(n^2+\tau)/p_{success}=O(n^2+\tau)$ expected iterations 
the optimum is found.
% Then, 
%  the population either finds or it consists of $\mu$ solutions selected 
% % uniformly at  random from $\{0,1\}^n$.
% For the rest of the proof we will refer 
% to any solutions which assign equal numbers of large jobs to each machine as 
% a 
% \emph{good} solution and to the rest as \emph{bad} solutions. If all the 
% solutions in a population are sampled uniformly at random, we will refer to 
% it as an initial population and to each solution in it 
% as an initial solution, regardless of the time they are sampled.
}

\begin{restatable}{theorem}{thageinglocal} \label{th:ageingG}
The \muoneeaageing optimises the \WittGeneralisedInstance class of instances in 
$O( n^2+\tau)$ steps in expectation for $\tau=\Omega(n^{1+c})$ for any 
arbitrarily small constant $c$ and $\mu=O(\log{n})$. \end{restatable}

\begin{proof}

With probability $\binom{s}{s/2} 2^{-s}{\color{black}=\Omega(1)}$, a randomly 
initialised solution assigns an equal number of large jobs to each machine. We will 
refer to such solutions as 
\textit{good} solutions.  For any initialised individual, the expected number of 
small jobs assigned to each machine is $(n-s)/2$. By Chernoff bounds 
{\color{black}(Theorem~\ref{thm-cher})}, the probability that less or more than 
$(n/2) \pm n^{3/5}$ small jobs are assigned to one machine is 
$e^{-\Omega(n^{1/5})}$. We will now show that if there is at least one good 
individual at initialisation, then in  $O(\mu \log{\mu})$ generations, 
the whole population will consist of good solutions {\color{black} with at least 
constant probability}. 

For a \textit{bad} (not good) initial solution, the difference between the number 
of large jobs in $M_1$ and $M_2$ is 
at least $2$ by definition. {\color{black}With $1-e^{-\Omega(n^{1/5})}$ 
probability,} the initial number of small jobs assigned to each machine will not 
differ by more than $O(n^{3/5})$ {\color{black}(event $\mathcal{E}_2$)} . 
Thus, the difference between the 
loads of the machines is at least $\Omega(1)$
for a bad solution, while less than $O(n^{-2/5})$ for a good one, both with 
overwhelmingly high probabilities. In order for a bad solution to have a better 
fitness than a good one, it is necessary that at least $\Omega(n)$ small jobs 
are moved from the fuller machine to the emptier machine which takes at least 
$\Omega(n/\log{n})$ time with overwhelmingly high probability since the 
probability of SBM flipping more than $\log{n}$ bits is exponentially small. 

If there are $k$ good solutions in the population, the probability that one of 
them is selected as parent is $k/\mu$ 
and the probability that SBM does not flip any bits and yields a copy is 
$(1-(1/n))^n>1/3$. Thus, the expected time until the good solutions take over 
the population is at most $3\mu \sum_{k=1}^{\mu}(1/k)\leq 3\mu \ln{\mu}$. The 
probability that the good solutions take over the population in at most 
$O(n^{3/5})$ iterations (before any bad 
solution is improved to a fitness value higher than the initial fitness of a 
good solution) is by Markov's inequality in the order of 
${\color{black}1-}O((\log{n})(\log{\log{n}})/n^{3/5})$. Thus, 
if the initial population contains at least one good solution, it takes over the 
population with probability $1-o(1)$. Moreover, since the discrepancy (the 
difference between the loads of the machines) of the accepted good solutions 
will not increase above $\Omega(n^{-2/5})$, SBM needs to move a linear number of 
small jobs in a single iteration to create a bad solution from a good one, which 
does not happen with probability at least $1-n^{-\Omega(n)}$. 

 Starting from a suboptimal good solution, the fitness improves by moving any 
small job from the fuller machine to the emptier machine without 
{\color{black}moving} any of the large jobs. The probability of such an event is 
at least $(n-s)/2n \cdot 
(1-1/n)^{n-1}= \Omega(1)$ considering that there are at least $(n-s)/2n$ small 
jobs on the fuller machine. Since the difference between the number of small 
jobs is in the order $O(n^{3/5})$, the small jobs will be distributed exactly 
evenly and the global optimum will be reached in expected $O(n^{3/5})$ 
iterations. 
{\color{black}Using Markov's inequality, we can bound the probability that 
global optimum will be reached in $O(n^{3/5})$ iterations by $\Omega(1)$}. 

{Since all necessary events described for finding the optimum have at least probability $\Omega(1)$, the probability that 
they occur in the same run is at least in the order of $\Omega(1)$ as well. 
Thus, all that remains is to find the expected time until the 
population consists of $\mu$ randomly initialised individuals given that one of the 
necessary events fails to occur and the global optimum is not found. This 
expectation is at most $L\cdot (2en^2+\tau) (1+o(1))$ due to
Lemma~\ref{lem:restart} with $L=\Omega(1)$ which is in turn due to 
Property~\ref{prop:levelwitt}.2. Thus, our expectation is 
$(1/p_{success})\cdot O(n^2+\tau) + O(n^{3/5})=O(n^2+\tau) $.}
\end{proof}

%[change below for static ageing and for ($\mu+1$ instead of 1+1:]

%If the algorithm always re-initialise from a local optima, then \mbox{\oneoneeaageing} can still find the global optimum by escaping from the locally optimal points, something that \oneoneea using restarts cannot do. W.o.p, the locally optimal individual does not improve and ages during each generation. When the age reaches the limit $\tau$, an offspring is created which separates the large jobs with probability at least $s/n \cdot (1-1/n)^{n-1}$. At the end of such generation, the parent (which ages at least $\tau$) dies with probability $1/2$ and the offspring (which also ages $\tau$) survives with the same probability, i.e., $1/2$. Now the offspring needs to stay safe by improving the fitness in the next step, which happens with probability at least $\Omega(n)/n \cdot (1-1/n)^{n-1}$ since there are a linear number of small jobs. As shown in Theorem 7.5 of \cite{Witt2005} (should be added to preliminaries) by using the Expected Multiplicative Distance Decrease method and under the assumption that the large jobs do not get mutated in a phase of length $cn$ (which happens with probability $(1-s/n)^{cn}$), in expected $O(n)$ steps enough small jobs will be moved to the emptier machine until both machine loads are equal. Therefore, the total time to reach the optimum is $O(n^2)$ according to Lemma 7.1 in \cite{Witt2005} (should be added in preliminaries).

Clearly the following corollary holds as $P_{\epsilon}^*$ is an instance of 
\WittGeneralisedInstance.
\begin{corollary}
The  \muoneeaageing optimises $P_{\epsilon}^*$ in $O(\mu n^2+\tau)$ steps in 
expectation for $\tau=\Omega(n^{{\color{black}1+c}})$ and $\mu=O(\log n)$. 
\end{corollary}

\section{$(1+\epsilon)$ Approximation Ratios} \label{sec:approx}
In this section we will show that both the \oneoneais and the \muoneeaageing are 
polynomial time approximation schemes, i.e., they guarantee solutions 
with makespan at most a $(1+\epsilon)$ factor worse than the optimal solution 
in expected time polynomial in the number of jobs $n$ and exponential only in 
$1/\epsilon$. Throughout this section we will refer to the largest $s:=\lceil 
2/\epsilon\rceil -1 \leq n/2$ jobs of the instance as \emph{large jobs} and the 
rest as \emph{small jobs}. Before we give the main results of the section, we 
present the properties of the \partition problem which will be used in the 
proofs.

\begin{property}\label{prop:approx}
 Let $I$ be any instance of the \partition problem with
%  $n$ jobs of weight $p_1\geq p_2 \geq \ldots \geq p_n$, 
 optimal makespan value 
$y^*$ and the set of local optima $\mathcal{L}$. Then,
% :=\{\ell_1, \ell_2, \ldots \}$ 
% indexed such that $f(\ell_i)\geq f(\ell_{i+1})$ for all $i \in 
% [|\mathcal{L}|-1]$. 
\begin{enumerate}
 \item For all $j\in\{s+1,\ldots,n\}$, $p_{j}\leq \frac{\epsilon}{2}
\sum_{i=1}^{n}p_i$.
  \item If there is at least one small job assigned to the fuller machine of a 
solution $x\in \mathcal{L}$, then $x$ is a $(1+\epsilon)$ approximation.
\item If $\sum_{i=s+1}^n p_i \geq \frac{1}{2} \sum_{i=1}^n p_i$, then all $x\in 
\mathcal{L}$ are $(1+\epsilon)$ approximations.
\item $|\{y\in ((1+\epsilon)\cdot y^*,\infty):\exists x\in \mathcal{L} \;\;\text{s.t.} \;\; 
f(x)=y  \}|\leq 2^{2/\epsilon}$.
\end{enumerate}
\end{property}
\begin{proof}
 Let $W:=\sum_{i=1}^{n}p_i$ denote the sum of the processing times of all 
jobs. We will prove the first statement by contradiction. Let us assume that 
$p_{s+1}>\frac{\epsilon}{2}W$. Thus, all $p_i$ for $i\in\{1,\ldots,s+1\}$ would 
satisfy $p_{i}>\frac{\epsilon}{2}W$, which implies $\sum_{i=1}^{s+1}p_i> 
\lceil\frac{2}{\epsilon}\rceil\frac{\epsilon}{2}W>W$, a contradiction.

The second statement follows from the first one, since if moving a small job 
from the fuller machine to the emptier one does not improve the makespan, 
then the makespan is at most $W/2+(1/2)(\epsilon/2)W$. A
trivial lower bound on $y^*$ is $W/2$, therefore the approximation ratio is at 
most $\left(\frac{W}{2}+\frac{1}{2}\frac{\epsilon}{2}W\right)/(W/2)\leq 
(1+\epsilon)$.

For the third statement, if $\sum_{i=s+1}^n p_i \geq \frac{1}{2} \sum_{i=1}^n 
p_i$, then any local optimum must have some small jobs on the fuller machine 
since the sum of the processing times of large jobs is smaller than the sum 
of the processing times of the small jobs. Thus, the claim follows from the 
second statement.

The last statement indicates that there are at most $2^{\epsilon/2}$ 
distinct makespan values among local optima which are not 
$(1+\epsilon)$ approximations. The second statement implies that if a local optimum is not a $(1+\epsilon)$ approximation, then it cannot have a small job on its fuller machine. Thus, in such a solution the makespan is exclusively determined by the configuration of the large jobs on the fuller machine. Since there 
are at most $\frac{\epsilon}{2}+1$ large jobs, the number of different 
configurations is $2^{\frac{\epsilon}{2}+1}$. Since the complement of each 
configuration has the same makespan, the statement follows.
\end{proof}

\subsection{Hypermutations}

% In this subsection
Theorem~\ref{thm:approxhyp} shows that the \oneoneais can efficiently find 
arbitrarily good constant approximations to any \partition instance.  
Before we state our main result, we will introduce the following helper 
lemma. 
% 
% \begin{theorem}\label{thm:hoeff}
%  Let $X_1,\ldots, X_n$ be independent or negatively correlated random 
% variables. Assume that each $X_i$ takes values in a real interval $[a_i,b_i]$ 
% of length $c_i=b_i-a_i$. Let $X=\sum_{i=1}^n X_i$. Then for all $\lambda>0$,
% \begin{align*}
%  \prob{X\geq E[X]+\lambda}&\leq 
% \emph{exp}\left(-\frac{2\lambda^2}{\sum_{i=1}^n c_{i}^2}\right),\\
% \prob{X\leq E[X]-\lambda}&\leq 
% \emph{exp}\left(-\frac{2\lambda^2}{\sum_{i=1}^n c_{i}^2}\right).
% \end{align*}
% 
% \end{theorem}
% 

\begin{restatable}{lemma}{funcflow}\label{lem:funcflow}
{\color{black}Let $x^{i} \in \{0,1 \}^{n}$ be the $i_{th}$ bitstring sampled by 
the hypermutation operator with the input bitstring $0^n$. For any 
arbitrary subset of indices $S\subset [n]$ and a function  $f(x):= \sum_{j \in 
S} x_j w_j $  with a set of non-negative weights $w_j$, $E[f(x^{i})]=
\frac{i}{n}\sum_{j \in S}  w_j $.} \end{restatable}

\begin{proof}
By linearity of expectation, $E[f(x^i)]= \sum_{j \in S}E[x_{j}^{i}] w_j $, 
where $x_{j}^{i}$ is the $j$th bit position of the $i$th bitstring sampled by 
the hypermutation operator. Considering that the initial bit value is  zero, 
$x_{j}^{i}=1$  only if its flip position is before $i$ (with probability $i/n$). 
Thus the expectation  
$E[x_{j}^{i}]$ is $i/n$ for all $j$ and the claim follows from moving the 
multiplicative factor of $i/n$ out of the sum in the expected function value 
$\sum_{j \in S} (i/n) w_j $. 
\end{proof}
We now state the main result of this section.
\begin{restatable}{theorem}{mainresulthyp}\label{thm:approxhyp}
The \oneoneais finds a $(1+\epsilon)$ approximation to any instance of 
\partition in at most
%\[
$n(\epsilon ^{-(2/\epsilon)-1})(1-\epsilon)^{-2} e^{3} 2^{2/\epsilon} + 2n^3 
2^{2/\epsilon} +2 n^3$
%\]  
fitness function evaluations in expectation for any 
$\epsilon=\omega(n^{-1/2})$. 
\end{restatable}

\begin{proof}
Let $X_t$ denote the current solution of the $(1+1)$~IA$^{hyp}$. We will assume that the 
algorithm stops as soon as it finds a $(1+\epsilon)$ approximation. The 
expected number of iterations where $X_t \notin \mathcal{L}$ is at most 
$2n^2(2^{2/\epsilon}+1)$ due to Property~\ref{prop:approx}.4 and 
Lemma~\ref{lem:local}.1. Moreover, due to Property~\ref{prop:approx}.3 and 
Lemma~\ref{lem:local}, an approximation is found in expected $2n^2$ iterations if 
$\sum_{i=s+1}^n p_i \geq \frac{1}{2} \sum_{i=1}^n p_i$. Therefore, for the rest 
of the proof we will assume that $\sum_{i=s+1}^n p_i < \frac{1}{2} \sum_{i=1}^n 
p_i$.

Now, we will bound the expected number of iterations where $X_t\in 
\mathcal{L}$.
Let $\confx$ denote the configuration of only the large jobs with 
the minimum makespan $\confy$. 
Note that this configuration might be different than the configuration of the 
large jobs in the optimal solution since it does not take the contribution of 
the small jobs to the makespan into account. However, $\confy \leq y^*$ since 
introducing the small jobs cannot improve the optimal makespan. 
% Consider the optimal configuration of large jobs {\color{green}$y^{*}:=y_{2^s}$ 
% } and 
% denote its makespan as $L$. 
Since both $\confx$  and its complementary 
bitstring have the same makespan, w.l.o.g., we will assume that the 
fuller machines of $X_t$ and $\confx$ are both $M_1$. 

According 
to Lemma~\ref{lem:hypinterv}, the $s\leq 2/\epsilon$ large jobs are assigned in
the same way as in $\confx$ between the $n(\epsilon-\epsilon^2)$th and 
$n-n(\epsilon-\epsilon^2)$th bit-flips with probability at least 
$\left(\epsilon-\epsilon^2 \right)^{2/\epsilon}e^{-1}$.
% \begin{align*}
% \left(\frac { 
% n(\epsilon-\epsilon^2)-\frac{2}{\epsilon}+1}{n-\frac{2}{\epsilon
% } +1}\right)^{\frac{2}{\epsilon}}&= \frac{\left(\epsilon-\epsilon^2 
% \right)^{2/\epsilon}}{\left(1-o(n^{-1/2}) \right)^{-2/\epsilon}}  &\geq 
% \frac{\left(\epsilon-\epsilon^2 \right)^{2/\epsilon} }{e},
% \end{align*}
% where we used the condition $\epsilon=\omega(n^{-1/2})$ to obtain the second 
% and the third expressions.
Due to Property~\ref{prop:approx}.2, $X_t$ 
does not assign any small jobs to $M_1$. 
Since $\sum_{i=s+1}^n p_i\leq W/2$, by 
the $n(\epsilon-\epsilon^2)$th bit-flip the expected total processing time of 
small jobs moved from $M_2$ to $M_1$ is at most $(\epsilon-\epsilon^2)W/2$ 
according to Lemma~\ref{lem:funcflow}. 
Due to Markov's inequality, with 
probability at least $1- ((\epsilon-\epsilon^2)W/(\epsilon)W)=\epsilon$, {\color{black}the 
sum of the processing times of the moved jobs} is {\color{black}at most} $\epsilon W /2$ and the makespan of the 
solution is at most $\confy+ \epsilon W /2$. Since $y^* \geq \confy$ and $y^*\geq 
W/2$, $(\confy+ \epsilon W /2)/y^* \leq (1+\epsilon)$. This 
implies that with probability 
\begin{align*}
p_{\approx} &\geq \frac{(\epsilon-\epsilon^2)^{2/\epsilon}}{e}\epsilon = 
\frac{\epsilon ^{(2/\epsilon)+1}}{e(1-\epsilon)^{-2/\epsilon}}
%\end{align*}   
%\begin{align*}
   =\frac{\epsilon 
^{(2/\epsilon)+1}(1-\epsilon)^{2}}{e(1-\epsilon)^{-((1/\epsilon)-1)2}} \geq 
\frac{\epsilon ^{(2/\epsilon)+1}(1-\epsilon)^2 }{e^{3}},
\end{align*}
a $(1+\epsilon)$ approximation is found unless an improvement is obtained before. 
Therefore, due to Property~\ref{prop:approx}.4, the expected number of 
iterations such that $X_t\in \mathcal{L}$ is at most $(1/\epsilon 
^{(2/\epsilon)+1})(1-\epsilon)^{-2} e^{3} 2^{2/\epsilon}$. Considering that 
hypermutation evaluates at most $n$ solutions per iteration, the expected 
number of fitness evaluations before a $(1+\epsilon)$ approximation is found is 
at most $n(\epsilon ^{-(2/\epsilon)-1})(1-\epsilon)^{-2} e^{3} 2^{2/\epsilon} + 2n^3 
2^{2/\epsilon} + 2n^3.$
\end{proof}
{\color{black}
In \cite{Witt2005} it is shown that the \oneoneea  achieves the same approximation ratio 
% for any constant $\epsilon$ 
as the \oneoneais while the leading term in the upper bound on its expected runtime is $n\log{(1/\epsilon)}2^{2/\epsilon}$. 
% rather than $n\epsilon^{-2/\epsilon}2^{2/\epsilon}+2^{2/\epsilon}n^3$. 
% lower leading constant of ${2^{(e \log{e}+e)\lceil 2/\epsilon \rceil \ln(4/\epsilon)+O(1/\epsilon)}}$ 
% on its runtime.
Although this result shows the power of EAs in general, for the \oneoneea to achieve this, a problem-specific restart schedule (where each run takes $O(n\log{(1/\epsilon)})$ expected iterations) has to be decided in advance where either the length of each run or the number of parallel runs depends on the desired approximation ratio. Consequently, such an algorithm loses application generality (e.g., if the ideal restart schedule for any constant $\epsilon$ is used, it would fail to optimise efficiently even \textsc{OneMax} or \textsc{LeadingOnes} since optimising these functions require $\Omega(n\log{n})$ and $\Omega(n^2)$ expected time respectively). In any case, if knowledge that the tackled problem is \partition was available, then using the \oneoneea would not be ideal in the first place.
On the other hand, the AIS does not require knowledge that the tackled problem is \partition, neither that a decision on the desired approximation ratio is made in advance for it to be reached efficiently. }

%\begin{proof}

%The probability of large jobs being assigned equally in number to both machines is $\binom{s}{s/2} 2^{-s}$.  As there might be some large jobs with the same weight on one machine, with probability $\binom{s}{s/2} 2^{-s}-2\cdot 2^{-s}$ at most $s/2$ moves are needed to switch the large jobs and put them in the \textit{right} place regarding their contribution to the global optimum. Each switch of two large jobs takes $O(n^2)$ steps in expectation. Considering that there are also equal number of small jobs in both machines after initialisation, such switches improve the fitness (do they? should be some conditions). After at most $s/2 \cdot O(n^2)$ steps in expectation, the large jobs will be in their \textit{right} places. Then, the optimum will be found in $O(n^2)$ in expectation.

%Without \textit{right} distribution of the large jobs in the initialisation, which happens with probability $1-\binom{s}{s/2} 2^{-s}$,  in expected $\Theta(\tau)+O(n^2)$ steps the algorithm will fall into a local optimum. In such case, after $\tau$ steps, the individual will die and a new solution will be initialised uniformly at random. 

%Since in expectation after $\binom{s}{s/2} 2^{-s}$ re-initialisations the algorithm starts from a solution with the same number of large jobs in both machines, the total expected time to find the global optimum will be

%\begin{equation*}
%\left(\binom{s}{s/2} \cdot 2^{-s} \right)^{-1} \cdot O(n^2+\tau) + O((s/2 \cdot n^2)+n^2)= O(n^2).
%\end{equation*}
 
%\end{proof}

\subsection{Ageing}

% In this subsection 

The following theorem shows that the \oneoneeaageing can find $(1+\epsilon)$ 
approximations. 
The proof follows the same ideas used to prove that RLS and the \oneoneea 
achieve a $(1+\epsilon)$ approximation if an appropriate restart strategy is put 
in place~\cite{Witt2005,NeumannWitt2010}. 
{\color{black} 
The correct restart strategy which allows RLS and the \oneoneea to approximate 
\partition requires a fair understanding of the problem. For instance, as discussed at the end of the previous section, for 
\partition, restarts have to happen roughly every $n \ln{(1/\epsilon)}$ 
iterations to achieve the smallest upper bound. However, the same restart strategy can cause exponential runtimes for 
easy problems if the expected runtime is in the order of $\Omega(n^c)$. Our 
parameter choices for the ageing operator are easier in the sense that no (or 
very limited) problem knowledge is required to decide upon its appropriate value 
and the choice mostly depends on the variation operator of the algorithm 
rather than on the problem at hand. For example, if the ageing operator is applied 
with RLS, a value of $\tau=n^{1+c}$ guarantees that with overwhelming probability 
no restart will ever happen unless the algorithm is trapped on a local optimum 
independent of the optimisation problem. If RLS$_{1,2, …, k}$ was applied, then a 
value of $\tau=n^{k+c}$ allows to draw the same conclusions. Concerning the 
\oneoneea, once the user has decided on their definition of local optimum 
(i.e., the maximum neighbourhood size for which the algorithm is expected to 
find a better solution before giving up), then the appropriate values for $\tau$ 
become obvious. In this sense, the algorithm using ageing works out “by itself” 
when it is stuck on a local optimum. 
}
We will show our result for the more restricted case of $\mu=1$ which suffices for the purpose of showing that ageing is efficient for the problem. 
In particular, the proof is considerably simplified with $\mu=1$ since we can use the success probability proven in~\cite{Witt2005} without any modification.
Furthermore, without proof, the upper bound on the runtime for 
arbitrary $\mu$ increases exponentially with respect to $\mu/\epsilon$. Our proof idea  is to use the 
success probability of the simple \oneoneea
without ageing and show that ageing automatically causes restarts whenever the  
\muoneeaageing fails to find the approximation {\color{black}by using only 1-bit 
flips}. Hence, a problem specific restart strategy is not necessary  to achieve the desired 
approximation. 
\begin{restatable}{theorem}{thmapproxage} \label{thm:approxage}

The \oneoneeaageing with $\tau=\Omega(n^{1+c})$ for any arbitrarily small 
constant $c>0$ finds a $(1+\epsilon)$ approximation to any instance of 
\partition in at most\\
$(2e n^2+\tau) 2^{(e \log{e}+e)\lceil 2/\epsilon \rceil \ln(4/\epsilon)+\lceil 
4/\epsilon \rceil -1}(1+o(1))$ fitness function evaluations in expectation for 
any $\epsilon \geq 4/n$. 
% This runtime is in the order $O(n^2+\tau)$ for any $\epsilon=\Omega(1)$.
\end{restatable}

\begin{proof}
We will start our proof by referring to Theorem~3 in~\cite{Witt2005} which lower 
bounds the probability that the \oneoneeaageing finds the $(1+\epsilon)$ 
approximation in $\lceil e n \log{4/\epsilon}\rceil$ iterations by $
p_{success}:=2^{-(e \log{e}+e)\lceil 2/\epsilon \rceil \ln(4/\epsilon) - \lceil 
2/\epsilon \rceil} $
for any $\epsilon \geq 4/n$.  We note here that since $\tau>\lceil e n 
\log{4/\epsilon}\rceil$, the ageing operator does not affect this probability.

Since we do not distinguish between an optimal solution and a $(1+\epsilon)$ 
approximation, Lemma~\ref{lem:restart}  and Property~\ref{prop:approx}.4 
imply that in at most $2^{s}(2e n^2+\tau)(1+o(1))$ expected time  either an 
approximation is found or a reinitialisation occurs. 
% bound the expected number 
% of iterations before ageing causes a 
% restart given that the $(1+\epsilon)$ approximation is not found. Multiplying 
% this bound with $p_{success}^{-1}$ will yield our expected time. We will use the 
% same  partitioning of the solution space into subspaces $A_i$ 
% for $i \in [2^s]$ as in the proof of Theorem~\ref{thm:approxhyp}. Following the 
% arguments from Theorem~\ref{thm:approxhyp}, for any solution in $A_i$,   it 
% takes at most $e n^2$ iterations in expectation to sample $x_{i}^{*}$. In the 
% worst case, it spends $\tau$ iterations at $x_{i}^{*}$ before improving at the 
% last step. In the worst case this happens in all $2^s$ phases, and the 
% \oneoneeaageing gets stuck at $x_{2^s -1}^{*}$ before the ageing causes 
% a restart. This yields an expected time less than $2^{s}(e n^2+\tau)$. 
Multiplying with $p_{success}^{-1}$ we obtain $(2e n^2+\tau) 2^{(e 
\log{e}+e)\lceil 2/\epsilon \rceil \ln(4/\epsilon) +\lceil 4/\epsilon \rceil 
-1}(1+o(1))$.
% since $s=\lceil 2/\epsilon \rceil -1$. 
% For any constant $\epsilon$, this runtime is in the order of $O(n^2+\tau)$. 
\end{proof}

\section{Conclusion}
To the best of our knowledge this is the first time that polynomial expected 
runtime guarantees of solution quality have been provided concerning AISs for a 
classical combinatorial optimisation problem. We presented a class of instances 
of \partition to illustrate how hypermutations and ageing can efficiently escape 
from local optima and provide optimal solutions where the standard bit mutations 
used by EAs get stuck for 
exponential time. Then we showed how this capability allows the AIS to achieve 
arbitrarily good $(1+\epsilon)$ approximations for any instance of \partition in 
{\color{black} expected time polynomial in the number of jobs $n$ and 
exponential only in $1/\epsilon$, i.e., in expected } polynomial time for any constant~$\epsilon$. In 
contrast to standard EAs and 
RLS that require {\color{black} problem specific number of } parallel runs or restart schemes to achieve such 
approximations, the AISs find them in a single run. The result is achieved in 
different ways. The ageing operator locates more promising basins of attraction 
by restarting the optimisation process after implicitly detecting 
 that it has found a local optimum. Hypermutations find improved 
approximate solutions efficiently by performing large jumps in the search space. 
Naturally, the proof strategy would also apply to the complete standard 
Opt-IA~AIS~\cite{CutelloTEVC,CorusOlivetoYazdani2017}  if the ageing parameter 
is set large enough, i.e., $\tau = \Omega(n^{1+c})$ for any arbitrarily small constant $c$. % 
\section*{References}
\bibliography{mybib2}

\end{document}